\documentclass{article}

\PassOptionsToPackage{numbers, compress}{natbib}

 \usepackage[preprint]{neurips_2026}

\usepackage[utf8]{inputenc} 
\usepackage[T1]{fontenc}    
\usepackage{hyperref}       
\usepackage{url}            
\usepackage{booktabs}       
\usepackage{amsfonts}       
\usepackage{amsmath}        
\usepackage{nicefrac}       
\usepackage{microtype}      
\usepackage{xcolor}         
\usepackage[breakable]{tcolorbox}  
\usepackage{amssymb, amsthm}
 \usepackage{placeins}
 \usepackage{mathtools}  
\newtheorem{theorem}{Theorem}
\newtheorem{assumption}{Assumption}

\DeclareMathOperator{\Var}{Var}
\DeclareMathOperator{\Cov}{Cov}
\DeclareMathOperator{\Corr}{Corr}
\DeclareMathOperator*{\E}{\mathbb{E}}

\title{SD-GRPO: Verifiable Segment Decomposition for Long-Form Vision-Language Generation}

\author{%
  Hyunwoong Kim\thanks{Corresponding author} \quad
  Seongeun Lee \quad
  Hannah Yun \quad
  Junhyun Park \quad
  Jonggwon Park \\
  DEEPNOID Inc., Seoul, South Korea \\
  \texttt{\{hwkim, selee, hnyun, jh\_park, jgpark\}@deepnoid.com} \\  
}

\begin{document}
\maketitle
\begin{abstract}

Group Relative Policy Optimization (GRPO) and its variants, originally developed for Large Language Models (LLMs), have recently been applied to Multimodal LLMs and produced strong results. 
However, their coarse-grained \emph{holistic} credit assignment from a single scalar advantage underfits vision-language (VL) tasks, where outputs are often long-form responses grounded in semantically rich images. 
To address this limitation, we exploit a structured signal that single-scalar formulations discard: the natural segmentation of long-form VL outputs.
Concretely, we propose Segment-Decomposed GRPO (SD-GRPO), which z-normalizes \emph{verifiable per-segment} rewards across the rollout group, yielding a vector of per-segment advantages in place of a single scalar.
We evaluate SD-GRPO across three settings spanning controlled and real-world long-form VL generation, organized by increasing semantic entanglement across segments.
On a controlled multi-panel dense-captioning task constructed from DOCCI, where segments are semantically independent, SD-GRPO consistently outperforms the GRPO baseline, with larger gains at higher segment counts.
Extending to a controlled multi-chart long-form VQA task constructed from MultiChartQA, we show both theoretically and empirically that rollout-level rewards suffer from cross-segment credit misattribution that scales with output length. SD-GRPO effectively bypasses this bottleneck, making it a structurally superior objective for reasoning models requiring long-form generation.
On a real-world scientific figure captioning task on the MMSci dataset, where subfigure captions share context across the figure, blending holistic and per-segment rewards further improves on both, suggesting per-segment normalization alone is insufficient when segments are semantically entangled.
Finally, by integrating SD-GRPO into Dr.~GRPO, we confirm that it can be applied to any GRPO framework with minimal implementation overhead to enhance long-form VL generation.

\end{abstract}

\begin{figure*}[t]
  \centering
  \includegraphics[width=\textwidth]{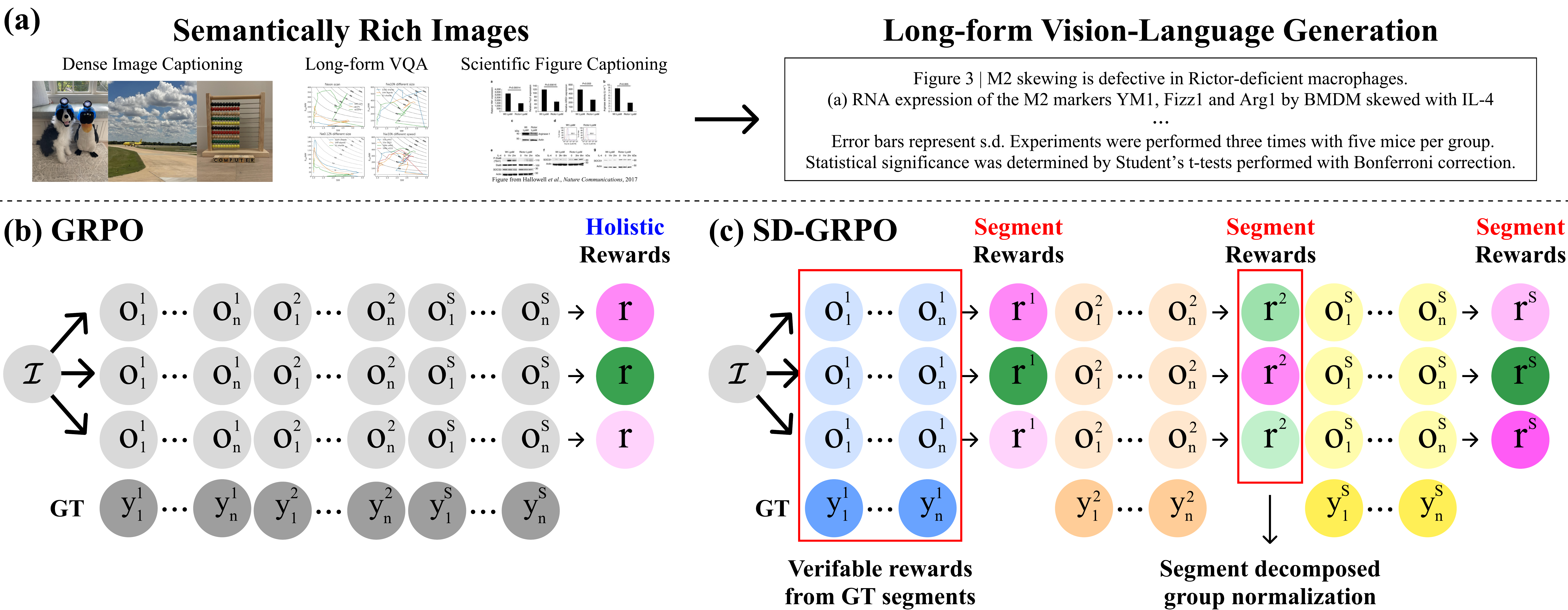}
    \caption{Illustration of SD-GRPO. (a) Vision-language generation from semantically rich images produces long-form generations. (b) Vanilla GRPO scores the entire rollout with one scalar reward. (c) SD-GRPO exploits the segmental structure of long-form generations, computing verifiable per-segment rewards for better credit assignment.}
  \label{fig:inrto}
\end{figure*}

\section{Introduction}
\label{sec:introduction}
Group Relative Policy Optimization (GRPO) is a leading framework within the Reinforcement Learning from Verifiable Rewards (RLVR) paradigm, frequently serving as the backbone for state-of-the-art large language models~\citep{guo2025deepseek,qwen25math}.
GRPO assigns a single scalar reward per output and derives one group-normalized advantage, a \emph{holistic} advantage, that is shared across all of its tokens. This design removes the critic model required by PPO~\citep{ppo} and yields stable, efficient training, and is a natural fit for problems whose trajectories are monolithic, such as math problems with a single verifiable answer. However, this efficiency comes at the cost of fine-grained credit assignment. Recent work in reasoning LLMs addresses this by partitioning chains of thought into segments and \emph{estimating} per-segment advantages; via Monte Carlo rollouts at segment boundaries~\citep{guosegment}, learned segment-level value functions~\citep{gong2026segmental}, retroactively matched anchor states~\citep{feng2025group}, reward shaping~\citep{fei2025self}, or outcome-conditioned baselines that stratify the rollout group by intermediate outcome~\citep{pavlenko2026blockwise}. Estimation is necessary in this setting because ground truth exists only at the trajectory end. Reasoning steps are also causally dependent, so per-step rewards cannot be cleanly computed.

Beyond LLMs, the GRPO framework has been applied to multimodal LLMs (MLLMs)~\citep{Huang2025VisionR1IR,chen2025r1v,Peng2025LMMR1E3}, achieving state-of-the-art results on their respective benchmarks. Most of these efforts focus on visual question answering (VQA) tasks, where a single verifiable answer maps cleanly onto GRPO's holistic reward. Yet long-form vision-language (VL) generation, which produces extended responses describing multiple aspects of an image, is also widespread, encompassing tasks such as dense image description~\citep{OnoeDocci2024,urbanek2024picture}, scientific figure captioning~\citep{li2024mmsci}, medical image report generation~\citep{johnson2019mimic}, and remote sensing image description~\citep{li2024vrsbench}. For these tasks, holistic credit assignment is poorly suited: a single advantage cannot distinguish erroneous tokens from those describing a correct observation, forcing every token in the rollout to receive identical credit regardless of its individual contribution. Fortunately, in these tasks the ground truth itself naturally decomposes into segments, each independently verifiable. Aligning the rollout's output with this segmentation enables \emph{per-segment} rewards to be \emph{computed} against ground truth rather than \emph{estimated}.

Building on this insight, we introduce Segment-Decomposed GRPO (\textbf{SD-GRPO}), an extension of GRPO with segment-level credit assignment for long-form VL generation. Each rollout output is partitioned into segments aligned with their ground truths; per-segment rewards are computed and z-normalized across the rollout group to yield a per-segment advantage vector, lifting GRPO's group-relative advantage from the trajectory to the segment level (Figure~\ref{fig:inrto}). Unlike prior segment-level methods, segment rewards are computed, therefore SD-GRPO requires no Monte Carlo rollouts, learned critics, or anchor-state matching. SD-GRPO becomes even more effective for reasoning models, where longer rollouts amplify cross-segment gradient misattribution; we show this both theoretically (Section~\ref{sec:residual}) and empirically (Section~\ref{sec:vqa-results}). SD-GRPO applies to any task whose output decomposes into verifiable segments and integrates into any GRPO variant with minimal cost.

We evaluate SD-GRPO extensively on long-form VL tasks, across three settings spanning controlled and real-world tasks, organized by increasing semantic entanglement across segments.
(1) On a controlled multi-panel dense-captioning task built from DOCCI~\citep{OnoeDocci2024}, where segments are semantically independent, SD-GRPO outperforms GRPO by 1.2-2.8 BERTScore-F1~\citep{zhang2019bertscore} points across two model sizes, with larger gains at higher segment counts; the same advantage holds when applied to a GRPO variant, Dr.~GRPO~\citep{liu2025understanding} (1.7-4.3 pp).
(2) On a controlled multi-chart long-form VQA task from MultiChartQA~\citep{zhu2025multichartqa}, SD-GRPO yields 1.0-2.3 pp accuracy improvements over GRPO in reasoning models, holding across three model scales and Dr.~GRPO (1.0-3.2 pp).
(3) On real-world scientific figure captioning from MMSci~\citep{li2024mmsci}, where subfigure captions share context across the figure, \emph{blending} holistic and per-segment rewards further improves on both, suggesting per-segment normalization alone is insufficient when segments are semantically entangled.
Our contributions are as follows:
\begin{enumerate}
    \item We identify long-form VL generation as a setting where per-segment rewards can be \emph{computed} directly against ground truth, and propose Segment-Decomposed GRPO (\textbf{SD-GRPO}), where \emph{verifiable per-segment} rewards are z-normalized across the rollout group to yield a vector of per-segment advantages.

    \item We evaluate SD-GRPO extensively on three long-form VL tasks spanning controlled and real-world settings. 
      SD-GRPO yields consistent improvements over GRPO, with the trend holding across multiple model scales and a GRPO variant (Dr.~GRPO).

    \item We show, both theoretically and empirically, that the benefit of per-segment credit assignment grows for reasoning models. 
    Furthermore, when segments share semantic context, we find that \emph{blending} holistic and per-segment rewards 
    outperforms either alone, suggesting that the optimal granularity of credit assignment depends on cross-segment structure.

\end{enumerate}

\section{Related Works} 
\label{sec:related_works}

\subsection{Segment-level credit assignment in RL}
\label{sec:related_segment}

A line of recent work has proposed segment-level credit assignment to improve trajectory-level reinforcement learning (RL) for reasoning LLMs. 
SPO \citep{guosegment} partitions a chain-of-thought (CoT) into contiguous segments using cutpoint-based boundaries, then estimates per-segment advantages through Monte Carlo rollouts at segment boundaries, either via chain-based estimation for short CoT or tree-based sampling for long CoT. GiGPO \citep{feng2025group} introduces a two-level advantage structure for multi-turn LLM agents. SPRO \citep{fei2025self} applies reward shaping to redistribute holistic rewards across sequential reasoning steps, and SAE \citep{gong2026segmental} learns segment-level advantage estimators directly. All of these methods operate in text-only domains: mathematical reasoning, code generation, or text-based agent environments, where sequential causal steps necessitate estimated advantages.

\subsection{RL fine-tuning of MLLMs}
\label{sec:related_mllm}
Reinforcement learning has been extended from LLMs to MLLMs, with several recent works applying GRPO and its variants to vision-language reasoning. Vision-R1~\citep{Huang2025VisionR1IR} applies GRPO to multimodal reasoning with verifiable rewards from multiple-choice and short-answer benchmarks. R1-V~\citep{chen2025r1v} demonstrates that GRPO transfers to MLLMs on counting, OCR, and visual math tasks where answers admit exact-match verification. LMM-R1~\citep{Peng2025LMMR1E3} extends this paradigm with a two-stage curriculum from text-only to multimodal reasoning. Across these methods, the reward signal is consistently a single trajectory-level scalar derived from a single verifiable answer; multiple-choice match, short-form QA accuracy, or numerical equality, and the holistic credit-assignment regime that GRPO inherits from text-only RLVR transfers cleanly.

\subsection{Long-form VL generation}
\label{sec:related_longform}
Long-form vision-language generation spans several application domains. Dense image description on DOCCI~\citep{OnoeDocci2024} and ImageInWords~\citep{urbanek2024picture} produces multi-sentence captions that decompose along visual sub-entities or attributes. Remote-sensing description on VRSBench~\citep{li2024vrsbench} and scientific figure captioning on SciCap~\citep{hsu2021scicap} and MMSci~\citep{li2024mmsci} likewise produce multi-segment outputs in which the reference itself is a structured concatenation of verifiable units. Despite the natural segmentation in these tasks, supervised fine-tuning on long-form references has remained the dominant training paradigm. RL approaches to long-form captioning have historically used trajectory-level surface metrics (CIDEr~\citep{vedantam2015cider}, BLEU~\citep{papineni2002bleu}, BERTScore) within REINFORCE-style methods such as Self-Critical Sequence Training (SCST)~\citep{rennie2017self}, treating each rollout's caption as a single unit and optimizing a single holistic reward~\citep{liu2017improved,li2018hybrid}. These works do not exploit the segment-level ground truth that the reference makes available, and apply the same scalar advantage to every token in the rollout.
\section{Method}
\label{sec:method}

\subsection{Preliminaries: Group Relative Policy Optimization}
\label{sec:method_grpo}

Let $\pi_\theta$ denote the MLLM policy with parameters $\theta$, and let $x$ be a VL input consisting of an image and a text prompt. For each $x$, GRPO~\citep{grpo} samples a group of $G$ rollouts $\{o_g\}_{g=1}^G$ from the old policy $\pi_{\theta_\text{old}}$, where each rollout $o_g = (o_{g,1}, \dots, o_{g,|o_g|})$ is a sequence of $|o_g|$ output tokens. A scalar reward $r_g = R(o_g)$ is assigned to each rollout, and a group-normalized \emph{holistic advantage} is computed:
\begin{equation}
\hat{A}_g = \frac{r_g - \mu_r}{\sigma_r + \delta}, \quad \text{where} \quad \mu_r = \tfrac{1}{G}\textstyle\sum_g r_g, \;\; \sigma_r^2 = \tfrac{1}{G}\textstyle\sum_g (r_g - \mu_r)^2,
\end{equation}
with a small $\delta > 0$ for numerical stability. The same advantage $\hat{A}_g$ is assigned to every token in rollout $g$. The policy is updated by maximizing a clipped surrogate objective with a per-token KL penalty against a fixed reference policy $\pi_\text{ref}$ (typically the supervised fine-tuned initialization):
\begin{equation}
\begin{aligned}
\mathcal{J}_\text{GRPO}(\theta) = \mathbb{E}_{\,x,\; o \sim \pi_{\theta_\text{old}}(O \mid x)} \Bigg[ \frac{1}{G} \sum_{g=1}^G \frac{1}{|o_g|} \sum_{t=1}^{|o_g|} \Big( &\min\!\left( \rho_{g,t}\, \hat{A}_g,\; \mathrm{clip}(\rho_{g,t}, 1{-}\epsilon, 1{+}\epsilon)\, \hat{A}_g \right) \\
&\;-\; \beta \cdot D_\text{KL}\!\left[ \pi_\theta \,\|\, \pi_\text{ref} \right] \Big) \Bigg],
\end{aligned}
\end{equation}
where $\rho_{g,t} = \pi_\theta(o_{g,t} \mid x, o_{g,<t}) / \pi_{\theta_\text{old}}(o_{g,t} \mid x, o_{g,<t})$ is the importance ratio against the rollout policy, $\epsilon$ is the clip range, $\beta$ is the KL coefficient, and the per-token KL estimator follows the unbiased low-variance form. 
The Dr.~GRPO variant~\citep{liu2025understanding} drops the $\sigma_r$ normalization (i.e., $\hat{A}_g = r_g - \mu_r$) to mitigate length and difficulty biases inherited from PPO.

\subsection{Long-form VL generation}
\label{sec:method_setting}

We formalize long-form VL tasks in which each output rollout admits a natural partition into $S$ contiguous segments,
\begin{equation}
o_g = \big(o_{g,1}, o_{g,2}, \dots, o_{g,S}\big),
\end{equation}
where each $o_{g,s}$ is a contiguous sub-sequence of tokens, and the reference is correspondingly structured into $S$ ground-truth segments $y^\star_1, \dots, y^\star_S$, each independently verifiable.

A per-segment reward function $R_\text{seg}$ scores each segment against its corresponding reference, yielding a vector of segment rewards for each rollout:
\begin{equation}
\mathbf{r}_g = \big(r_{g,1}, \dots, r_{g,S}\big), \quad r_{g,s} = R_\text{seg}\big(o_{g,s},\, y^\star_s\big).
\end{equation}
$R_\text{seg}$ is any reference-based metric. By default, the holistic reward used by standard GRPO is obtained by applying the same metric to the entire rollout against the concatenated reference,
\begin{equation}
r_g = R_\text{seg}\big(o_g,\, y^\star_\text{concat}\big), \qquad y^\star_\text{concat} = \mathrm{concat}\big(y^\star_1, \dots, y^\star_S\big),
\end{equation}
treating the rollout as a single unit without reference to its segmentation. For completely independent segments (the multi-panel image captioning task in Section~\ref{sec:docci-exp}) and VQA (Section~\ref{sec:vqa-exp}), we instead set $r_g = \tfrac{1}{S} \sum_s r_{g,s}$, a mean-of-segments holistic reward (\textit{mean-GRPO}).

\subsection{Segment-Decomposed GRPO (SD-GRPO)}
\label{sec:method_sdgrpo}

SD-GRPO replaces GRPO's scalar advantage with a vector of per-segment advantages, each computed by group-normalization \emph{within its own segment index}:
\begin{equation}
\hat{A}_{g,s} = \frac{r_{g,s} - \mu_{r,s}}{\sigma_{r,s} + \delta}, \quad \mu_{r,s} = \tfrac{1}{G}\textstyle\sum_g r_{g,s}, \;\; \sigma_{r,s}^2 = \tfrac{1}{G}\textstyle\sum_g \big(r_{g,s} - \mu_{r,s}\big)^2.
\label{eq:sdgrpo_advantage}
\end{equation}
Each token inherits the advantage of the segment it belongs to: for a token at position $t$ in rollout $g$, let $s(t)$ denote its segment index; the SD-GRPO objective replaces $\hat{A}_g$ with $\hat{A}_{g,s(t)}$ in the per-token surrogate,
\begin{equation}
\begin{aligned}
\mathcal{J}_\text{SD-GRPO}(\theta) = \mathbb{E}_{\,x,\; o \sim \pi_{\theta_\text{old}}(O \mid x)} \Bigg[ \frac{1}{G} \sum_{g=1}^G \frac{1}{|o_g|} \sum_{t=1}^{|o_g|} \Big( &\min\!\left( \rho_{g,t}\, \hat{A}_{g,s(t)},\; \mathrm{clip}(\rho_{g,t}, 1{-}\epsilon, 1{+}\epsilon)\, \hat{A}_{g,s(t)} \right) \\
&\;-\; \beta \cdot D_\text{KL}\!\left[ \pi_\theta \,\|\, \pi_\text{ref} \right] \Big) \Bigg].
\end{aligned}
\end{equation}
This is a direct extension of GRPO's group-relative advantage from the trajectory level to the segment level. SD-GRPO requires no additional rollouts, learned critics, or anchor-state matching, and reduces to standard GRPO when $S = 1$.

\subsection{Cross-segment residual amplification in long-form generation}
\label{sec:residual}

Holistic GRPO updates segment $s$ with the rollout-level advantage
$\hat{A}_g$, while SD-GRPO updates it with the segment-level advantage
$\hat{A}_{g,s}$. Decomposing the holistic advantage as
\begin{equation}
  \hat{A}_g \;=\; \hat{A}_{g,s} \;+\; \eta_{g,s},
  \qquad
  \eta_{g,s} \,\coloneqq\, \hat{A}_g - \hat{A}_{g,s},
\end{equation}
allows us to isolate $\eta_{g,s}$ as the \textit{credit misattribution} induced in GRPO.
This identity splits the holistic per-segment gradient into the SD-GRPO update
$\hat{A}_{g,s}\,\boldsymbol{\psi}_{g,s}$ plus the
\emph{cross-segment residual}
\begin{equation}
  \boldsymbol{\xi}_{g,s} \,\coloneqq\, \eta_{g,s}\,\boldsymbol{\psi}_{g,s},
  \qquad
  \boldsymbol{\psi}_{g,s} \,\coloneqq\, \sum_{t \in o_{g,s}}
  \nabla\log\pi_\theta(o_{g,t}\mid o_{g,<t},\,x),
\end{equation}
where $\boldsymbol{\psi}_{g,s}$ is the segment's accumulated score.
SD-GRPO sets $\boldsymbol{\xi}_{g,s} = \mathbf{0}$ by construction;
the question is how large this residual is under holistic GRPO.

Both advantages are z-normalized within their respective groups, so
$\E[\hat{A}_g] = \E[\hat{A}_{g,s}] = 0$ and
$\Var(\hat{A}_g) = \Var(\hat{A}_{g,s}) = 1$. Let
$\kappa_s \coloneqq \Corr(\hat{A}_g, \hat{A}_{g,s})$ measure how well the holistic and segment-level advantages agree.

\begin{assumption}[Non-negative intra-segment score correlation]
\label{ass:score_corr}
For every segment $s$ and every $t, t' \in o_{g,s}$,
$\E\bigl[\nabla\log\pi_{g,t} \cdot \nabla\log\pi_{g,t'}\bigr] \geq 0$.
\end{assumption}
Tokens within a segment share policy parameters and conditioning
context, so we expect their score functions to be aligned in
expectation.

\begin{theorem}[Cross-segment residual under holistic GRPO]
\label{thm:residual}
Under Assumption~\ref{ass:score_corr},
\begin{equation}
\E\bigl[\|\boldsymbol{\xi}_{g,s}\|^2\bigr]
\;=\;
\E\bigl[\eta_{g,s}^2\,\|\boldsymbol{\psi}_{g,s}\|^2\bigr],
\end{equation}
where the two factors satisfy
\begin{equation}
\E[\eta_{g,s}^2] \,=\, 2(1 - \kappa_s),
\qquad
\E\bigl[\|\boldsymbol{\psi}_{g,s}\|^2\bigr]
\,\geq\, |o_{g,s}|\cdot\E\bigl[\|\nabla\log\pi_{g,t}\|^2\bigr].
\end{equation}
Under SD-GRPO, $\eta_{g,s} \equiv 0$ and so
$\boldsymbol{\xi}_{g,s} = \mathbf{0}$.
\end{theorem}

\begin{proof}
The identity follows from $\boldsymbol{\xi}_{g,s} = \eta_{g,s}\,\boldsymbol{\psi}_{g,s}$.
For $\eta_{g,s}$,
\[
\E[\eta_{g,s}^2]
= \Var(\hat{A}_g) + \Var(\hat{A}_{g,s})
- 2\,\Cov(\hat{A}_g, \hat{A}_{g,s})
= 2(1 - \kappa_s).
\]
For $\boldsymbol{\psi}_{g,s}$, expand the squared norm and apply
Assumption~\ref{ass:score_corr} to drop the cross-token terms:
\[
\E\bigl[\|\boldsymbol{\psi}_{g,s}\|^2\bigr]
= \sum_{t} \E\bigl[\|\nabla\log\pi_{g,t}\|^2\bigr]
+ \sum_{t \neq t'} \E\bigl[\nabla\log\pi_{g,t}\cdot\nabla\log\pi_{g,t'}\bigr]
\geq |o_{g,s}|\cdot\E\bigl[\|\nabla\log\pi_{g,t}\|^2\bigr],
\]
writing $\E[\|\nabla\log\pi_{g,t}\|^2]$ for the per-token average over
$t \in o_{g,s}$.
\end{proof}

The two factors clarify when the residual is large.
$\E[\eta_{g,s}^2] = 2(1-\kappa_s)$ measures how much holistic and
segment-level advantages disagree. $\E[\|\boldsymbol{\psi}_{g,s}\|^2]$ grows
at least linearly in $|o_{g,s}|$, so longer segments accumulate more
residual mass per update. This makes reasoning models a structurally favorable case for SD-GRPO.
A reasoning model interleaves CoT tokens with each
segment's answer tokens, inflating $|o_{g,s}|$ and therefore
$\E[\|\boldsymbol{\psi}_{g,s}\|^2]$; the residual misattribution by
holistic GRPO grows correspondingly, while SD-GRPO removes it from the
gradient by construction. We confirm empirically in
Section~\ref{sec:vqa-results} that the effect of SD-GRPO is emphasized for reasoning models.

\subsection{Blended advantage for entangled segments}
\label{sec:method_blend}

SD-GRPO eliminates the credit misattribution $\eta_{g,s}$ as residual, which is appropriate when
each segment's reward depends only on its own tokens. However, when segments are
semantically entangled (e.g., subfigure captions share context across the figure), $\eta_{g,s}$ also carries informative cross-segment signal that SD-GRPO discards.

To preserve a controlled fraction of this signal, we consider a
\emph{blended} advantage that combines the per-segment and holistic
signals,
\begin{equation}
\label{eq:blend}
\hat{A}^\text{blend}_{g,s} \;=\; \alpha \cdot \hat{A}_{g,s} \;+\; (1 - \alpha) \cdot \hat{A}_g,
\end{equation}
where $\hat{A}_g$ is the group-normalized holistic advantage of the
concatenated-reference reward (Section~\ref{sec:method_setting}). The
coefficient $\alpha \in [0, 1]$ trades off targeted per-segment credit
against global coherence: $\alpha = 1$ recovers SD-GRPO, $\alpha = 0$
recovers holistic GRPO. We use $\alpha = 0.5$ in our scientific figure
captioning experiments (Section~\ref{sec:mmsci-exp}).

\section{Experiments}
\label{sec:experiments}

\subsection{Multi-panel dense image captioning}
\label{sec:docci-exp}

\textbf{Task and setting.}
A composite of $S$ side-by-side panels is fed to the model, which emits $S$ numbered descriptions in a single rollout under controlled multi-turn decoding (Appendix~\ref{app:docci-prompts}).
Panels share visual context but receive per-panel rewards, so credit must propagate across panels within a sampling group.
We compare two advantage variants: Holistic GRPO (scalar advantage from the mean-of-panel reward) and SD-GRPO (per-panel advantages from per-panel rewards).
Each variant is run under both $\sigma$-normalized GRPO and Dr.~GRPO with all hyperparameters matched across conditions.

\textbf{Dataset.}
We tile $S \in \{3, 5, 7\}$ DOCCI~\citep{OnoeDocci2024} training images side-by-side ($224{\times}224$ per panel) without replacement, so every image in the 9{,}647-image training set appears once per epoch (3{,}215 / 1{,}929 / 1{,}378 composites for $S=3/5/7$); the 5{,}000-image official test set yields 1{,}666 / 1{,}000 / 714 evaluation composites.
Each panel's reference is the first three sentences of the corresponding DOCCI caption.

\textbf{Metric.}
We use per-panel BERTScore-F1 (baseline-rescaled) with \texttt{deberta-xlarge-mnli} as the reward. Because each panel is an independent source image with its own reference, per-panel scoring is well-posed for this task; we report the per-panel mean BERTScore-F1 in Table~\ref{tab:docci-main}, computed on the full test set at the final checkpoint of each run.
ROUGE-L, METEOR~\citep{banerjee2005meteor}, CIDEr, and BLEU-1/2/3/4 are also measured as cross-metric checks.

\textbf{Implementation.}
We train Qwen3-VL-2B/4B-Instruct~\citep{bai2025qwen3} with AdamW~\citep{loshchilov2017decoupled} (weight decay $0.01$) at a flat learning rate of $5\times 10^{-6}$, gradient clipping at norm $1.0$, KL coefficient $0.01$, PPO clip $\epsilon=0.2$, global batch size $8$, group size $G=8$, and bf16 mixed precision, for one image-epoch (402 / 241 / 172 gradient steps for $S=3/5/7$), following standard practice for compute-constrained RL training.
Each panel is decoded with a per-segment cap of $180$ tokens.

\subsection{Multi-chart long-form VQA}
\label{sec:vqa-exp}

\textbf{Task and setting.}
A group of $C$ chart panels and $S{=}4$ questions are fed to the model, which emits $S$ answer segments across an $S$-turn user--assistant conversation; each segment corresponds to the model's answer to one question.
Chart panels are attached only on the first user turn; subsequent turns carry only the next question and inherit the panel images from the conversation history.
We again compare Holistic GRPO (scalar advantage from the mean-segment reward) and SD-GRPO (per-segment advantages).
In addition, we vary a reasoning axis: in \emph{direct} mode the model emits the answer directly, while in \emph{reasoning} mode it is prompted to emit a short reasoning turn before each answer (visualization in Figure~\ref{fig:result}, prompts in Appendix~\ref{app:vqa-prompts}).
Each variant is run under both $\sigma$-normalized GRPO and Dr.~GRPO~\citep{liu2025understanding} with matching hyperparameters.

\textbf{Dataset.}
We adapt MultiChartQA~\citep{zhu2025multichartqa} into long-form VQA by combining the four questions per chart group into a single $S{=}4$ segment response, yielding $382$ training and $68$ evaluation chart groups.
This produces $1{,}528$ training questions ($1{,}063$ short single-value, $465$ compound semicolon-separated multi-fact) and $272$ evaluation questions ($183$ short, $89$ compound). Chart panels are resized so the longest edge is at most $336$~px.

\textbf{Metric.}
We compute exact-match for short targets ($\le 6$ tokens, no semicolon) and ROUGE-L F1 for compound or longer targets as the training reward.
For evaluation, we report mean per-segment exact-match accuracy in Table~\ref{tab:vqa-main}, computed on the full evaluation set at the final checkpoints.

\textbf{Implementation.}
We train Qwen3-VL-2B/4B/8B-Instruct with AdamW (weight decay $0.01$) at a flat learning rate of $5\times 10^{-6}$, gradient clipping at norm $1.0$, KL coefficient $0$, PPO clip $\epsilon=0.2$, global batch size $8$, group size $G=8$, and bf16 mixed precision, for $3$ epochs.
Each answer segment is capped at $20$ tokens, with an additional $64$-token reasoning turn before each answer in~\emph{reasoning} mode.

\subsection{Composite scientific figure captioning}
\label{sec:mmsci-exp}

\textbf{Task and setting.}
A composite scientific figure containing $S$ labeled subfigures $(a),(b),(c),\ldots$ is fed to the model, which emits all $S$ subcaptions in a single autoregressive pass, separated by line-start label markers (Appendix~\ref{app:mmsci-prompts}, visualization in Figure~\ref{fig:appendix}).
Unlike previous experiments, $S$ varies across examples (range $2$--$7$), showcasing the natural variable-$S$ extension of SD-GRPO.
In addition to Holistic GRPO (scalar reward from concatenated subcaptions) and SD-GRPO (per-section rewards), we evaluate a \emph{Blend} ($0.5\cdot\text{SD-GRPO}+0.5\cdot\text{Holistic}$ at the token-advantage level) to account for potential semantic entanglement between subcaptions. For ablation, we include \textit{mean-GRPO} that replaces the concat-level holistic reward with the mean per-section reward.

\textbf{Dataset.}
The official MMSci~\citep{li2024mmsci} release was unavailable, so we rebuild the dataset using the authors' released crawler, keeping all five \textit{Nature Communications} categories~\citep{hallowell2017mtorc2,liu2013photolatently}. We then filter to figures whose captions exhibit a clean monotonic $(a)(b)(c)\ldots$ subcaption structure, retaining $39{,}949$ training and $400$ evaluation figures under the official train/dev partition ($144$K and $1{,}418$ subcaptions, respectively). The training-set distribution over panel counts $S$ is $\{2\!:\!29.5\%,\ 3\!:\!24.2\%,\ 4\!:\!21.5\%,\ 5\!:\!11.6\%,\ 6\!:\!9.0\%,\ 7\!:\!4.4\%\}$.

\textbf{Metric.}
We blend two complementary signals, SciBERT~\citep{beltagy2019scibert} for token-level scientific-term meaningfulness and BARTScore~\citep{yuan2021bartscore} for sequence-level adequacy on the multi-sentence subcaptions, as the reward. Specifically, $r = 0.5\cdot\mathrm{SciBERT\text{-}F1} + 0.5\cdot\mathrm{BARTScore}_{[0,1]}$, where $\mathrm{SciBERT\text{-}F1}$ is BERTScore-style matching between SciBERT layer-9 token embeddings and $\mathrm{BARTScore}_{[0,1]}$ is BARTScore log-likelihood linearly mapped from $[-5, 0]$ to $[0, 1]$.
To ensure correct output format, we use a format gate where format-invalid responses have reward floored to $-1$.
We report both views of $r$ in Table~\ref{tab:mmsci-train} for fairness across credit-assignment schemes: \emph{concat $r$} (the scalar Holistic GRPO directly optimizes) and \emph{section-mean $r$} (mean of per-section rewards). Both are computed on the full evaluation set at the final checkpoints. We additionally report two GPT-4o text-only judge views~\citep{hurst2024gpt} in Table~\ref{tab:mmsci-judge}: (1) an absolute rater that scores each candidate against the human reference on accuracy, completeness, and fluency (1.0--5.0), and (2) a head-to-head pairwise judge. Detailed implementations are available in Appendix~\ref{app:mmsci-pairwise}.

\begin{table}
  \caption{Multi-panel DOCCI dense-captioning results across segment counts $S$, model scale, and GRPO variant. Values are per-panel mean BERTScore-F1 averaged over three seeds. $\Delta = \text{SD-GRPO} - \text{Holistic}$ (positive indicates SD-GRPO improvement). SD-GRPO outperforms Holistic in every configuration, with gains growing as $S$ increases. Detailed results in Appendix~\ref{app:docci-results}.}
  \label{tab:docci-main}
  \centering
  \small
  \begin{tabular*}{\linewidth}{@{\extracolsep{\fill}}lccccccccc@{}}
    \toprule
    & & & \multicolumn{3}{c}{GRPO} & \multicolumn{3}{c}{Dr.~GRPO} \\
    \cmidrule(lr){4-6}\cmidrule(lr){7-9}
    Model & $S$ & Base & Holistic & SD-GRPO & $\Delta$ & Holistic & SD-GRPO & $\Delta$ \\
    \midrule
    Qwen3-VL-2B & 3 & $0.261$ & $0.370$ & $\mathbf{0.383}$ & $+0.013$ & $0.296$ & $\mathbf{0.322}$ & $+0.026$ \\
                & 5 & $0.226$ & $0.354$ & $\mathbf{0.368}$ & $+0.014$ & $0.270$ & $\mathbf{0.308}$ & $+0.038$ \\
                & 7 & $0.192$ & $0.332$ & $\mathbf{0.357}$ & $+0.025$ & $0.252$ & $\mathbf{0.295}$ & $+0.043$ \\
    \midrule
    Qwen3-VL-4B & 3 & $0.242$ & $0.365$ & $\mathbf{0.380}$ & $+0.015$ & $0.263$ & $\mathbf{0.280}$ & $+0.017$ \\
                & 5 & $0.219$ & $0.341$ & $\mathbf{0.353}$ & $+0.012$ & $0.237$ & $\mathbf{0.262}$ & $+0.025$ \\
                & 7 & $0.209$ & $0.323$ & $\mathbf{0.351}$ & $+0.028$ & $0.228$ & $\mathbf{0.253}$ & $+0.025$ \\
    \bottomrule
  \end{tabular*}
\end{table}

\textbf{Implementation.}
We train Qwen3-VL-4B/8B-Instruct with AdamW (weight decay $0.01$) at a flat learning rate of $5\times 10^{-6}$, gradient clipping at norm $1.0$, KL coefficient $0.04$ (higher than the standard $0.01$, for stability with the format gate), PPO clip $\epsilon{=}0.2$, global batch size $64$, group size $G{=}8$, and bf16 mixed precision, for one image-epoch ($624$ gradient steps).
Images are capped at $200{,}704$ pixels, max prompt length is $4{,}096$ tokens, and max response length is $1{,}024$ tokens.

\begin{table}
  \caption{Multi-chart long-form VQA results. Values are mean exact-match accuracy across 3 seeds. In reasoning mode, the model is instructed to emit a short reasoning turn before each answer. 
  Bold marks the best values in each model size.
  Detailed results available in Appendix~\ref{app:vqa-results}.}
  \label{tab:vqa-main}
  \centering
  \small
  \begin{tabular*}{\linewidth}{@{\extracolsep{\fill}}llccccccc@{}}
    \toprule
    & & & \multicolumn{3}{c}{GRPO} & \multicolumn{3}{c}{Dr.~GRPO} \\
    \cmidrule(lr){4-6}\cmidrule(lr){7-9}
    Model & Mode & Base & Holistic & SD-GRPO & $\Delta$ & Holistic & SD-GRPO & $\Delta$ \\
    \midrule
    Qwen3-VL-2B & Direct    & $25.4$ & $29.5$ & $29.1$          & $-0.4$ & $29.2$          & $31.1$ & $+1.9$ \\
                & Reasoning & $23.2$ & $30.6$          & $\mathbf{31.6}$ & $+1.0$ & $29.4$          & $\mathbf{31.3}$ & $+1.9$ \\
    \midrule
    Qwen3-VL-4B & Direct    & $30.5$ & $33.3$          & $34.1$ & $+0.8$ & $34.8$ & $33.1$          & $-1.7$ \\
                & Reasoning & $32.4$ & $35.3$          & $\mathbf{37.6}$ & $+2.3$ & $35.7$          & $\mathbf{38.9}$ & $+3.2$ \\
    \midrule
    Qwen3-VL-8B & Direct    & $34.2$ & $39.1$ & $37.6$          & $-1.5$ & $37.5$ & $37.2$          & $-0.3$ \\
                & Reasoning & $36.8$ & $39.1$          & $\mathbf{41.3}$ & $+2.2$ & $40.2$          & $\mathbf{41.2}$ & $+1.0$ \\
    \bottomrule
  \end{tabular*}
\end{table}

\begin{figure}[t]
  \centering
  \includegraphics[width=\textwidth]{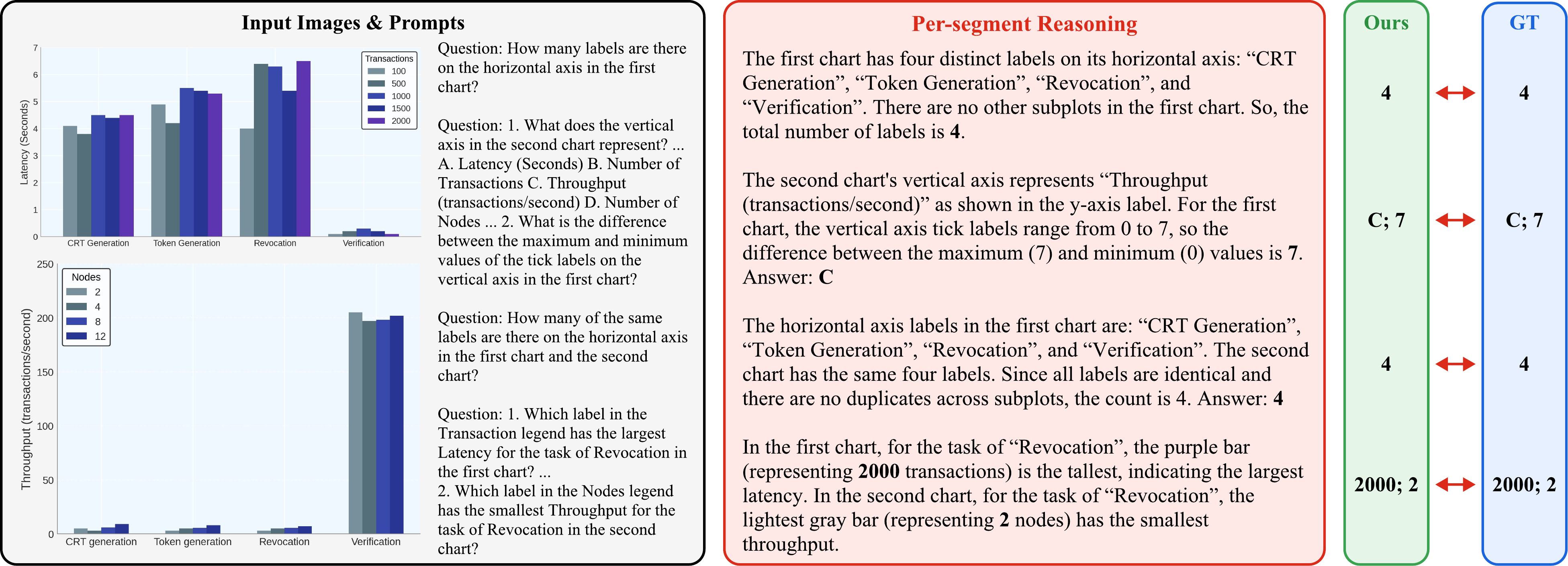}
 \caption{Reasoning trajectories from our SD-GRPO-trained model on multi-chart long-form VQA. The model learns to correctly identify the referenced chart, read the relevant values, and execute the required arithmetic before producing each answer.}
  \label{fig:result}
\end{figure}

\section{Results}
\label{sec:results}

\subsection{Multi-panel dense image captioning}
\label{sec:docci-results}

Table~\ref{tab:docci-main} shows SD-GRPO outperforms its Holistic counterpart across every setting, yielding $1.2$--$2.8$ absolute BERTScore points on the standard-normalized GRPO baseline and $1.7$--$4.3$ points under Dr.~GRPO~\citep{liu2025understanding}. The gap is consistent across all conditions and generally grows with segment count $S$: as segment counts increase the signals are diluted even further when averaged into a single scalar. Cross-metric results on BLEU, METEOR, ROUGE-L, and CIDEr (Appendix~\ref{app:docci-results}) show SD-GRPO generates outputs of generally higher textual quality across the standard NLG metrics.

\subsection{Multi-chart long-form VQA}
\label{sec:vqa-results}

Table~\ref{tab:vqa-main} reports per-segment exact-match accuracy averaged over three seeds.
All RL-trained variants substantially exceed the untrained baselines (Base) across model scales, indicating that the algorithmic gap is on top of a clear RL-training effect. SD-GRPO consistently outperforms Holistic GRPO in the reasoning regime, yielding $+1.0$ to $+3.2$ absolute points.
The advantage holds under both $\sigma$-normalized GRPO and Dr.GRPO and at all three model scales, indicating per-segment credit assignment as the source rather than interaction with the standardization step. Models trained with SD-GRPO produce valid, meaningful reasoning CoT before answering the questions (Figure~\ref{fig:result}).
In direct mode, by contrast, the comparison is mixed-sign ($-1.7$ to $+2.3$).
We attribute the mode asymmetry to our theory in Section~\ref{sec:residual}. The short, single-token answers in direct mode naturally bound the cross-segment residual; conversely, the extensive token count of a reasoning chain amplifies this noise, making SD-GRPO's neutralization of credit misattribution structurally load-bearing.
Consistent with this, the gain further concentrates on compound (long, multi-fact) segments, where per-segment reward variance and token spans are both large (Appendix~\ref{app:vqa-results}, Table~\ref{app:vqa-compound}). 
Reasoning-mode runs also consistently outperform direct-mode runs at every scale, and SD-GRPO's reliable advantage in this regime highlights its particular importance for reasoning models. 
Moreover, SD-GRPO tends to produce longer reasoning chains than Holistic GRPO (Appendix~\ref{app:vqa-results}, Table~\ref{app:vqa-reasoning-tokens}), consistent with per-segment advantages rewarding reasoning that yields correct per-segment answers.

\subsection{Composite scientific figure captioning}
\label{sec:mmsci-results}

The two pure variants exhibit complementary failure modes: Holistic GRPO under-rewards per-section quality, while SD-GRPO under-rewards figure-level coherence (Table~\ref{tab:mmsci-train}). In this regime of high semantic entanglement, the \textit{Blend} variant emerges as the superior choice, as it wins both concat reward (Holistic's training reward) and per-section reward (SD-GRPO's training reward).
Cross-seed standard deviations are tight, and Blend's section-mean lead over Holistic exceeds the seed std by $3$–$6\sigma$ at both scales (Appendix~\ref{app:mmsci-train-std}).
The GPT-4o text-only judge results in Table~\ref{tab:mmsci-judge} show Blend's single-rater grade is highest at both scales, and the head-to-head pairwise judge ranks Blend ahead of both pure variants by $+8.5$ to $+17.5$ percentage points overall. These results are further corroborated by traditional NLP metrics (ROUGE-L, BERTScore-F1, CIDEr) in Appendix~\ref{app:mmsci-cross-metrics}.
Blend is also robust to the mixing weight (Appendix~\ref{app:mmsci-alpha}).
Furthermore, we run a \textit{mean-GRPO} ablation that swaps Holistic's concat-level scalar for the section-mean scalar (Appendix~\ref{app:mmsci-holistic-avg}), and observe it falls short of concat-Holistic and falls well short of SD-GRPO and Blend. This demonstrates isolating segment-specific gradients is more critical than simply averaging granular rewards across the entire sequence, consistent with our framing of credit misattribution.

\begin{table}[!t]
  \caption{Composite scientific figure captioning main results. Values are averaged across 3 seeds. Concat reward is $r{=}0.5\cdot\mathrm{SciBERT\text{-}F1}+0.5\cdot\mathrm{BARTScore}_{[0,1]}$ on the concatenated response (reward of Holistic GRPO); Section-mean reward is per-section mean of $r$ (reward of SD-GRPO).}
  \label{tab:mmsci-train}
  \centering
  \small
  \begin{tabular*}{\linewidth}{@{\extracolsep{\fill}}lcccccccc@{}}
    \toprule
    & \multicolumn{4}{c}{Concat reward} & \multicolumn{4}{c}{Section-mean reward} \\
    \cmidrule(lr){2-5}\cmidrule(lr){6-9}
    Model & Base & Holistic & SD-GRPO & Blend & Base & Holistic & SD-GRPO & Blend \\
    \midrule
    Qwen3-VL-4B & $0.399$ & $0.416$ & $0.412$ & $\mathbf{0.420}$ & $0.382$ & $0.392$ & $\mathbf{0.404}$ & $\mathbf{0.404}$ \\
    Qwen3-VL-8B & $0.401$ & $0.419$ & $0.418$ & $\mathbf{0.427}$ & $0.383$ & $0.399$ & $\mathbf{0.411}$ & $\mathbf{0.411}$ \\
    \bottomrule
  \end{tabular*}
\end{table}

\begin{table}[!t]
  \caption{Composite scientific figure captioning results with GPT-4o text-only judge. Preference = head-to-head preference margin (Blend's win\% $-$ opponent's win\%), position-randomized; positive means Blend wins. Detailed results in Appendix~\ref{app:mmsci-pairwise}.}
  \label{tab:mmsci-judge}
  \centering
  \small
  \begin{tabular*}{\linewidth}{@{\extracolsep{\fill}}lccccc@{}}
    \toprule
    & \multicolumn{3}{c}{Grading (1.0--5.0, $\uparrow$)} & \multicolumn{2}{c}{Preference: Blend wins by ($\uparrow$)} \\
    \cmidrule(lr){2-4}\cmidrule(lr){5-6}
    Model & Holistic & SD-GRPO & Blend & vs Holistic & vs SD-GRPO \\
    \midrule
    Qwen3-VL-4B & $2.30$ & $2.30$ & $\mathbf{2.37}$ & $+10.5\%$ & $+11.2\%$ \\
    Qwen3-VL-8B & $2.40$ & $2.36$ & $\mathbf{2.45}$ & $+8.5\%$  & $+17.5\%$ \\
    \bottomrule
  \end{tabular*}
\end{table}

\section{Conclusion}
\label{sec:conclusion}

In this study, we propose Segment Decomposed GRPO (SD-GRPO), a variant of GRPO that replaces the rollout-level holistic advantage with per-segment z-normalized advantages for long-form VL tasks.
We identified a cross-segment residual that Holistic GRPO injects into each segment's gradient, a residual whose magnitude grows with segment length and is eliminated by SD-GRPO's construction.
Across three experimental settings, our empirical results consistently validate this analysis. When segments are semantically entangled, a simple blend of segment and holistic advantages preserves segment-level credit while restoring global coherence, outperforming both pure variants. Together, these results position localized credit assignment as a structurally favorable direction for the RL fine-tuning of long-form VL generation, particularly for reasoning models where extensive rollouts otherwise amplify the credit misattribution effect.

\paragraph{Limitations and future work.}
Our experiments are restricted to tasks with an explicit segmental structure. Applying SD-GRPO on tasks without natural segmentation would require an upstream LLM-based segmenter. Furthermore, our reasoning-mode results are limited to the multi-chart VQA task. The lead of SD-GRPO over GRPO persists under reasoning mode for other tasks, but absolute scores fall below direct mode because the policy fails to produce meaningful CoT (Appendix~\ref{app:think-results}). Looking forward, we plan to extend SD-GRPO to agentic VLMs for long-form VL tasks, where output segments naturally bound agentic steps.
\clearpage
{
\small
\bibliographystyle{plainnat}
\bibliography{ref}
}

\newpage
\appendix

\section{Implementation Details}
\subsection{Additional implementation details across experiments}
During training, rollouts are sampled at temperature $T=1$; evaluation uses greedy decoding ($T=0$).

\subsection{Computational cost}
\label{app:comp_cost}
\begin{table}[h]
\centering
\small
\caption{Computational cost of training.}
\begin{tabular}{lcccc}
\toprule
Model & GPUs & Memory / GPU & Total Time & GPU-hours \\
\midrule
Multi-panel dense image captioning & 8 $\times$ H200 & 70\,GB & 62h & 496h \\
Multi-chart long-form VQA & 8 $\times$ H200 & 80\,GB & 48h & 384h \\
Composite scientific figure captioning & 8 $\times$ H200 & 90\,GB & 207h & 1,656h \\
\bottomrule
\end{tabular}
\label{tab:compute_cost}
\end{table}

\section{Multi-panel dense image captioning: Prompts and decoding format}
\label{app:docci-prompts}

\paragraph{System prompt.}
The system prompt is topic-free and parameterized by the pool size $S$:

\begin{tcolorbox}[breakable, title=System prompt,
                  fonttitle=\bfseries\small,
                  colback=gray!5, colframe=black,
                  colbacktitle=black, coltitle=white]
\small\ttfamily\noindent
You will see a composite image with \{S\} panels arranged left to right. Panels are numbered 1 to \{S\} from left to right.\\
For each panel, write three short descriptive sentences on one numbered line.\\
Format:\\
Answer 1) <your three sentences for panel 1>\\
Answer 2) <your three sentences for panel 2>\\
\dots\\
Answer \{S\}) <your three sentences for panel \{S\}>
\end{tcolorbox}

\paragraph{User turn.}
The user turn contains two content blocks: the rendered $S$-panel composite image, followed by the literal text \texttt{"Describe each panel."}.

\paragraph{Controlled multi-turn decoding.}
The assistant response is generated piecewise across $S$ panels. For each panel $s = 1, \ldots, S$, the trainer deterministically appends the marker \texttt{"Answer s) "} to the context and samples until the model emits "\textbackslash n". Markers are always forced, never sampled, and are masked from the loss; reward and advantage are computed only over the sampled tokens. This guarantees the clean per-panel token boundaries that SD-GRPO's panel-level advantage shaping requires, and stays format-proof.

\paragraph{Resulting response structure.}
After the loop completes, the assembled assistant response under direct mode (\texttt{max\_output\_tokens}=180) has the form

\begin{tcolorbox}[breakable, title={Assistant response},
                  fonttitle=\bfseries\small,
                  colback=gray!5, colframe=black,
                  colbacktitle=black, coltitle=white]
\small\ttfamily\noindent
Answer 1) <model-generated text for panel 1>\\
Answer 2) <model-generated text for panel 2>\\
\dots\\
Answer S) <model-generated text for panel S>
\end{tcolorbox}

\section{Multi-panel dense image captioning: Detailed results}
\label{app:docci-results}
This section complements the headline BERTScore-F1 results of Section~\ref{sec:docci-exp} with cross-metric numbers across the standard image-captioning suite (Tables~\ref{tab:docci-std},~\ref{tab:docci-xmetric}).
Values are computed on the full $5{,}000$-image DOCCI test set at the final checkpoint of each run, identical to the protocol used for Table~\ref{tab:docci-main}.
The SD-GRPO-over-Holistic ordering observed in Table~\ref{tab:docci-main} is preserved across all seven metrics in $10$ of $12$ (Model, $S$, Algorithm) conditions; the two exceptions occur at $S{=}3$ under $\sigma$-normalized GRPO, where Holistic edges out SD-GRPO on B-1, B-2, and CIDEr while SD-GRPO retains its lead on B-3, B-4, METEOR, and ROUGE-L.

\begin{table}[t]
  \caption{Multi-panel DOCCI dense-captioning results across segment counts $S$, model scale, and GRPO variant. Values are BERTScore-F1 (per-panel mean, baseline-rescaled), mean $\pm$ std over three seeds. SD-GRPO outperforms Holistic GRPO in every cell.}
  \label{tab:docci-std}
  \centering
  \small
  \setlength{\tabcolsep}{4pt}
  \begin{tabular*}{\linewidth}{@{\extracolsep{\fill}}lccccccc@{}}
    \toprule
    & & & \multicolumn{2}{c}{GRPO} & \multicolumn{2}{c}{Dr.~GRPO} \\
    \cmidrule(lr){4-5}\cmidrule(lr){6-7}
    Model & $S$ & Base & Holistic & SD-GRPO & Holistic & SD-GRPO \\
    \midrule
    Qwen3-VL-2B & 3 & $0.261$ & $0.370_{\pm 0.004}$ & $\mathbf{0.383}_{\pm 0.001}$ & $0.296_{\pm 0.001}$ & $\mathbf{0.322}_{\pm 0.002}$ \\
                & 5 & $0.226$ & $0.354_{\pm 0.001}$ & $\mathbf{0.368}_{\pm 0.001}$ & $0.270_{\pm 0.001}$ & $\mathbf{0.308}_{\pm 0.002}$ \\
                & 7 & $0.192$ & $0.332_{\pm 0.005}$ & $\mathbf{0.357}_{\pm 0.001}$ & $0.252_{\pm 0.002}$ & $\mathbf{0.295}_{\pm 0.001}$ \\
    \midrule
    Qwen3-VL-4B & 3 & $0.242$ & $0.365_{\pm 0.012}$ & $\mathbf{0.380}_{\pm 0.012}$ & $0.263_{\pm 0.002}$ & $\mathbf{0.280}_{\pm 0.001}$ \\
                & 5 & $0.219$ & $0.341_{\pm 0.002}$ & $\mathbf{0.353}_{\pm 0.015}$ & $0.237_{\pm 0.000}$ & $\mathbf{0.262}_{\pm 0.002}$ \\
                & 7 & $0.209$ & $0.323_{\pm 0.010}$ & $\mathbf{0.351}_{\pm 0.000}$ & $0.228_{\pm 0.002}$ & $\mathbf{0.253}_{\pm 0.001}$ \\
    \bottomrule
  \end{tabular*}
\end{table}

\begin{table}
  \caption{Cross-metric DOCCI results across $S$, model scale, and GRPO variant. Values are mean $\pm$ std over three seeds. Per-panel mean for BLEU-1/2/3/4, METEOR, and ROUGE-L; corpus-level for CIDEr (its TF-IDF weighting is defined corpus-wide).}
  \label{tab:docci-xmetric}
  \centering
  \scriptsize
  \setlength{\tabcolsep}{4pt}
  \begin{tabular*}{\linewidth}{@{\extracolsep{\fill}}llllccccccc@{}}
    \toprule
    Model & $S$ & Algo & Method & B-1 & B-2 & B-3 & B-4 & METEOR & R-L & CIDEr \\
    \midrule
    Qwen3- & 3 & GRPO     & Holistic & $\mathbf{25.8}_{\pm 0.9}$ & $\mathbf{14.2}_{\pm 0.5}$ & $7.6_{\pm 0.2}$ & $4.1_{\pm 0.1}$ & $20.0_{\pm 0.3}$ & $26.2_{\pm 0.4}$ & $\mathbf{13.3}_{\pm 1.2}$ \\
    VL-2B    &   &          & SD-GRPO  & $25.1_{\pm 1.6}$ & $14.1_{\pm 0.8}$ & $\mathbf{7.7}_{\pm 0.4}$ & $\mathbf{4.2}_{\pm 0.2}$ & $\mathbf{20.3}_{\pm 0.6}$ & $\mathbf{27.3}_{\pm 0.1}$ & $12.0_{\pm 1.4}$ \\
                &   & Dr.~GRPO & Holistic & $26.3_{\pm 0.3}$ & $12.2_{\pm 0.1}$ & $5.5_{\pm 0.0}$ & $2.8_{\pm 0.0}$ & $18.8_{\pm 0.1}$ & $23.1_{\pm 0.1}$ & $14.2_{\pm 0.5}$ \\
                &   &          & SD-GRPO  & $\mathbf{26.5}_{\pm 0.4}$ & $\mathbf{13.0}_{\pm 0.1}$ & $\mathbf{6.0}_{\pm 0.0}$ & $\mathbf{3.1}_{\pm 0.0}$ & $\mathbf{19.1}_{\pm 0.1}$ & $\mathbf{24.3}_{\pm 0.2}$ & $\mathbf{14.6}_{\pm 0.5}$ \\
    \cmidrule{2-11}
                & 5 & GRPO     & Holistic & $23.7_{\pm 1.5}$ & $12.7_{\pm 0.8}$ & $6.7_{\pm 0.4}$ & $3.6_{\pm 0.2}$ & $18.5_{\pm 0.6}$ & $25.8_{\pm 0.2}$ & $9.4_{\pm 1.3}$ \\
                &   &          & SD-GRPO  & $\mathbf{25.4}_{\pm 1.3}$ & $\mathbf{14.0}_{\pm 0.6}$ & $\mathbf{7.4}_{\pm 0.3}$ & $\mathbf{4.0}_{\pm 0.1}$ & $\mathbf{19.8}_{\pm 0.5}$ & $\mathbf{26.8}_{\pm 0.1}$ & $\mathbf{10.8}_{\pm 1.3}$ \\
                &   & Dr.~GRPO & Holistic & $24.4_{\pm 0.2}$ & $10.9_{\pm 0.1}$ & $4.7_{\pm 0.1}$ & $2.4_{\pm 0.0}$ & $17.1_{\pm 0.1}$ & $22.4_{\pm 0.1}$ & $10.3_{\pm 0.4}$ \\
                &   &          & SD-GRPO  & $\mathbf{25.2}_{\pm 0.2}$ & $\mathbf{12.1}_{\pm 0.2}$ & $\mathbf{5.4}_{\pm 0.1}$ & $\mathbf{2.7}_{\pm 0.1}$ & $\mathbf{17.9}_{\pm 0.1}$ & $\mathbf{23.9}_{\pm 0.1}$ & $\mathbf{10.6}_{\pm 0.2}$ \\
    \cmidrule{2-11}
                & 7 & GRPO     & Holistic & $21.4_{\pm 1.5}$ & $11.0_{\pm 0.8}$ & $5.4_{\pm 0.6}$ & $2.8_{\pm 0.3}$ & $16.9_{\pm 0.7}$ & $24.4_{\pm 0.5}$ & $6.5_{\pm 1.1}$ \\
                &   &          & SD-GRPO  & $\mathbf{24.3}_{\pm 1.0}$ & $\mathbf{13.1}_{\pm 0.5}$ & $\mathbf{6.8}_{\pm 0.2}$ & $\mathbf{3.6}_{\pm 0.1}$ & $\mathbf{18.9}_{\pm 0.4}$ & $\mathbf{26.0}_{\pm 0.1}$ & $\mathbf{8.5}_{\pm 0.6}$ \\
                &   & Dr.~GRPO & Holistic & $22.8_{\pm 0.1}$ & $9.8_{\pm 0.0}$ & $4.0_{\pm 0.0}$ & $2.1_{\pm 0.0}$ & $15.9_{\pm 0.0}$ & $21.7_{\pm 0.1}$ & $7.2_{\pm 0.2}$ \\
                &   &          & SD-GRPO  & $\mathbf{23.3}_{\pm 0.6}$ & $\mathbf{10.9}_{\pm 0.3}$ & $\mathbf{4.7}_{\pm 0.2}$ & $\mathbf{2.4}_{\pm 0.1}$ & $\mathbf{16.6}_{\pm 0.3}$ & $\mathbf{23.1}_{\pm 0.0}$ & $\mathbf{7.8}_{\pm 0.8}$ \\
    \midrule
    Qwen3- & 3 & GRPO     & Holistic & $\mathbf{29.6}_{\pm 1.9}$ & $16.0_{\pm 0.8}$ & $8.5_{\pm 0.3}$ & $4.5_{\pm 0.1}$ & $21.9_{\pm 0.8}$ & $25.7_{\pm 1.3}$ & $\mathbf{18.4}_{\pm 2.2}$ \\
    VL-4B      &   &          & SD-GRPO  & $28.8_{\pm 3.7}$ & $16.0_{\pm 1.7}$ & $\mathbf{8.7}_{\pm 0.7}$ & $\mathbf{4.7}_{\pm 0.3}$ & $\mathbf{22.1}_{\pm 1.8}$ & $\mathbf{27.0}_{\pm 1.1}$ & $16.9_{\pm 3.9}$ \\
                &   & Dr.~GRPO & Holistic & $24.5_{\pm 0.3}$ & $10.2_{\pm 0.1}$ & $4.2_{\pm 0.1}$ & $2.1_{\pm 0.0}$ & $17.4_{\pm 0.2}$ & $19.7_{\pm 0.1}$ & $12.7_{\pm 0.2}$ \\
                &   &          & SD-GRPO  & $\mathbf{25.6}_{\pm 0.7}$ & $\mathbf{11.0}_{\pm 0.3}$ & $\mathbf{4.6}_{\pm 0.1}$ & $\mathbf{2.3}_{\pm 0.0}$ & $\mathbf{18.1}_{\pm 0.4}$ & $\mathbf{20.3}_{\pm 0.1}$ & $\mathbf{13.9}_{\pm 0.5}$ \\
    \cmidrule{2-11}
                & 5 & GRPO     & Holistic & $28.1_{\pm 0.8}$ & $14.6_{\pm 0.7}$ & $7.2_{\pm 0.8}$ & $3.8_{\pm 0.4}$ & $20.3_{\pm 0.7}$ & $25.0_{\pm 0.7}$ & $14.7_{\pm 0.2}$ \\
                &   &          & SD-GRPO  & $\mathbf{29.1}_{\pm 2.3}$ & $\mathbf{15.4}_{\pm 0.8}$ & $\mathbf{7.5}_{\pm 0.3}$ & $\mathbf{4.0}_{\pm 0.2}$ & $\mathbf{21.3}_{\pm 0.7}$ & $\mathbf{25.4}_{\pm 1.7}$ & $\mathbf{15.1}_{\pm 2.7}$ \\
                &   & Dr.~GRPO & Holistic & $23.9_{\pm 0.1}$ & $9.3_{\pm 0.1}$ & $3.8_{\pm 0.0}$ & $1.9_{\pm 0.0}$ & $16.5_{\pm 0.1}$ & $19.3_{\pm 0.0}$ & $10.5_{\pm 0.2}$ \\
                &   &          & SD-GRPO  & $\mathbf{25.2}_{\pm 0.0}$ & $\mathbf{10.5}_{\pm 0.0}$ & $\mathbf{4.4}_{\pm 0.0}$ & $\mathbf{2.2}_{\pm 0.0}$ & $\mathbf{17.6}_{\pm 0.0}$ & $\mathbf{20.3}_{\pm 0.1}$ & $\mathbf{11.4}_{\pm 0.2}$ \\
    \cmidrule{2-11}
                & 7 & GRPO     & Holistic & $27.3_{\pm 1.0}$ & $13.7_{\pm 0.6}$ & $6.4_{\pm 0.7}$ & $3.3_{\pm 0.4}$ & $19.4_{\pm 0.6}$ & $23.9_{\pm 1.1}$ & $11.2_{\pm 0.5}$ \\
                &   &          & SD-GRPO  & $\mathbf{28.6}_{\pm 0.3}$ & $\mathbf{15.4}_{\pm 0.2}$ & $\mathbf{7.8}_{\pm 0.1}$ & $\mathbf{4.1}_{\pm 0.0}$ & $\mathbf{20.8}_{\pm 0.2}$ & $\mathbf{26.1}_{\pm 0.1}$ & $\mathbf{12.4}_{\pm 0.5}$ \\
                &   & Dr.~GRPO & Holistic & $23.0_{\pm 0.1}$ & $8.6_{\pm 0.1}$ & $3.4_{\pm 0.1}$ & $1.8_{\pm 0.0}$ & $15.7_{\pm 0.1}$ & $19.1_{\pm 0.1}$ & $8.3_{\pm 0.2}$ \\
                &   &          & SD-GRPO  & $\mathbf{24.2}_{\pm 0.2}$ & $\mathbf{9.9}_{\pm 0.1}$ & $\mathbf{4.0}_{\pm 0.1}$ & $\mathbf{2.0}_{\pm 0.0}$ & $\mathbf{16.7}_{\pm 0.2}$ & $\mathbf{20.0}_{\pm 0.1}$ & $\mathbf{9.0}_{\pm 0.4}$ \\
    \bottomrule
  \end{tabular*}
\end{table}

\section{Multi-chart long-form VQA: Prompts and conversation format}
\label{app:vqa-prompts}

\paragraph{System prompt.}
The system prompt is mode-dependent. In \textbf{Direct} mode the model is instructed to emit only the short answer:

\begin{tcolorbox}[breakable, title={System prompt---direct mode},
                  fonttitle=\bfseries\small,
                  colback=gray!5, colframe=black,
                  colbacktitle=black, coltitle=white]
\small\ttfamily\noindent
Answer questions about a group of 2--3 charts. Reply with just the short answer (a number, percentage, option letter, or short phrase; for two-value questions, separate with '; '), nothing else.
\end{tcolorbox}

\noindent In \textbf{Reasoning} mode, the prompt scopes a 1--2 sentence reasoning step before the model is asked for its final answer:

\begin{tcolorbox}[breakable, title={System prompt---reasoning mode},
                  fonttitle=\bfseries\small,
                  colback=gray!5, colframe=black,
                  colbacktitle=black, coltitle=white]
\small\ttfamily\noindent
Answer questions about a group of 2--3 charts. Each turn has two steps: first you think out loud in 1--2 short sentences (identify the referenced chart, read the relevant values, and state your reasoning). You will then be asked for the final short answer and must reply with just the short answer (a number, percentage, option letter, or short phrase; for two-value questions, separate with '; '), nothing else.
\end{tcolorbox}

\noindent In Reasoning mode the system message is followed by two in-context few-shot examples illustrating the reason-then-answer protocol; we omit them here for brevity.

\paragraph{User turn.}
For section $s=1$, the user turn carries the $C\!\le\!3$ chart panels as image attachments followed by the section's question text. For sections $s=2,\ldots,S$, the user turn carries only the question text; images are inherited from the conversation history.

\paragraph{Multi-turn decoding.}
Unlike the controlled forced-marker decoding used for DOCCI (Appendix~\ref{app:docci-prompts}), each section here is sampled as a real user--assistant turn in a multi-turn conversation. For each section $s = 1, \ldots, S$:
\begin{itemize}
\item In \textbf{Direct} mode, the model samples a short answer turn (max $20$ tokens, stops at \texttt{eos}).
\item In \textbf{Reasoning} mode, the model first samples a reasoning turn (max $64$ tokens), the trainer appends a fixed user follow-up \texttt{"Now give your final short answer only."}, and the model then samples the short answer turn (max $20$ tokens).
\end{itemize}
After each section, the assistant's outputs are committed to the conversation history before the next section's user question is appended. Per-section rewards are computed on the answer text only; the reasoning text is loss-shaped by the same per-section advantage but does not contribute to the reward.

\paragraph{Resulting response structure.}
Under Direct mode, each section is one user--assistant turn ($\texttt{max\_thinking\_tokens}=0$, $\texttt{max\_output\_tokens}=20$):

\begin{tcolorbox}[breakable, title={Conversation---direct mode},
                  fonttitle=\bfseries\small,
                  colback=gray!5, colframe=black,
                  colbacktitle=black, coltitle=white]
\small\noindent
\textbf{User:} [chart panel images] + <section 1 question>\\
\textbf{Assistant:} <model-generated answer for section 1>\\
\textbf{User:} <section 2 question>\\
\textbf{Assistant:} <model-generated answer for section 2>\\
\dots\\
\textbf{User:} <section S question>\\
\textbf{Assistant:} <model-generated answer for section S>
\end{tcolorbox}

\noindent and under Reasoning mode ($\texttt{max\_thinking\_tokens}=64$, $\texttt{max\_output\_tokens}=20$), each section produces two assistant turns (reasoning, then answer) separated by the fixed answer-prompt user message:

\begin{tcolorbox}[breakable, title={Conversation---reasoning mode},
                  fonttitle=\bfseries\small,
                  colback=gray!5, colframe=black,
                  colbacktitle=black, coltitle=white]
\small\noindent
\textbf{User:} [chart panel images] + <section 1 question>\\
\textbf{Assistant:} <model-generated reasoning for section 1>\\
\textbf{User:} Now give your final short answer only.\\
\textbf{Assistant:} <model-generated answer for section 1>\\
\textbf{User:} <section 2 question>\\
\textbf{Assistant:} <model-generated reasoning for section 2>\\
\textbf{User:} Now give your final short answer only.\\
\textbf{Assistant:} <model-generated answer for section 2>\\
\dots
\end{tcolorbox}

\section{Multi-chart long-form VQA: Detailed results}
\label{app:vqa-results}
For completeness, Table~\ref{tab:vqa-std} reports the same multi-chart long-form VQA results as Table~\ref{tab:vqa-main} with per-seed standard deviations across three seeds.
Table~\ref{app:vqa-reasoning-tokens} reports the mean number of reasoning tokens emitted per segment in reasoning mode, averaged over three seeds. SD-GRPO produces longer reasoning chains than Holistic GRPO under matched backbone and optimizer in $5$ of $6$ cells.
On the compound (multi-fact) subset of the evaluation set---answers containing semicolons, $89$ of $272$ total---reward is ROUGE-L F1 against the canonical answer-key string. Table~\ref{app:vqa-compound} reports the compound-subset reward at the final checkpoint of each run, averaged over three seeds. The SD-GRPO advantage is positive in $11$ of $12$ (model, mode, optimizer) cells; the single exception is 8B under GRPO in direct mode ($-3.6$).

\begin{table}
  \caption{Multi-chart long-form VQA results. Values are mean $\pm$ std exact-match accuracy across 3 seeds. All models are Qwen3-VL-Instruct variants; in reasoning mode, the model is instructed to emit a short reasoning turn before each answer.}
  \label{tab:vqa-std}
  \centering
  \small
  \begin{tabular*}{\linewidth}{@{\extracolsep{\fill}}llccccc@{}}
    \toprule
    & & & \multicolumn{2}{c}{GRPO} & \multicolumn{2}{c}{Dr.~GRPO} \\
    \cmidrule(lr){4-5}\cmidrule(lr){6-7}
    Model & Mode & Base & Holistic & SD-GRPO & Holistic & SD-GRPO \\
    \midrule
    Qwen3-VL-2B & Direct    & $25.4$ & $\mathbf{29.5}{\scriptstyle\pm0.8}$ & $29.1{\scriptstyle\pm0.7}$          & $29.2{\scriptstyle\pm1.0}$          & $\mathbf{31.1}{\scriptstyle\pm1.7}$ \\
                & Reasoning & $23.2$ & $30.6{\scriptstyle\pm1.4}$          & $\mathbf{31.6}{\scriptstyle\pm1.1}$ & $29.4{\scriptstyle\pm0.3}$          & $\mathbf{31.3}{\scriptstyle\pm2.3}$ \\
    \midrule
    Qwen3-VL-4B & Direct    & $30.5$ & $33.3{\scriptstyle\pm0.5}$          & $\mathbf{34.1}{\scriptstyle\pm0.5}$ & $\mathbf{34.8}{\scriptstyle\pm1.2}$ & $33.1{\scriptstyle\pm0.6}$          \\
                & Reasoning & $32.4$ & $35.3{\scriptstyle\pm0.8}$          & $\mathbf{37.6}{\scriptstyle\pm0.8}$ & $35.7{\scriptstyle\pm0.8}$          & $\mathbf{38.9}{\scriptstyle\pm0.2}$ \\
    \midrule
    Qwen3-VL-8B & Direct    & $34.2$ & $\mathbf{39.1}{\scriptstyle\pm1.2}$ & $37.6{\scriptstyle\pm1.1}$          & $\mathbf{37.5}{\scriptstyle\pm0.3}$ & $37.2{\scriptstyle\pm1.4}$          \\
                & Reasoning & $36.8$ & $39.1{\scriptstyle\pm1.9}$          & $\mathbf{41.3}{\scriptstyle\pm1.9}$ & $40.2{\scriptstyle\pm1.2}$          & $\mathbf{41.2}{\scriptstyle\pm2.6}$ \\
    \bottomrule
  \end{tabular*}
\end{table}

\begin{table}
  \caption{Mean reasoning tokens per segment in reasoning mode, mean $\pm$ std across 3 seeds.}
  \label{app:vqa-reasoning-tokens}
  \centering
  \small
  \begin{tabular*}{\linewidth}{@{\extracolsep{\fill}}lccccc@{}}
    \toprule
    & & \multicolumn{2}{c}{GRPO} & \multicolumn{2}{c}{Dr.~GRPO} \\
    \cmidrule(lr){3-4}\cmidrule(lr){5-6}
    Model & Base & Holistic & SD-GRPO & Holistic & SD-GRPO \\
    \midrule
    Qwen3-VL-2B & $16.5$ & $28.5{\scriptstyle\pm 6.0}$ & $\mathbf{35.9}{\scriptstyle\pm 0.9}$ & $21.4{\scriptstyle\pm 12.1}$ & $\mathbf{31.0}{\scriptstyle\pm 9.0}$ \\
    Qwen3-VL-4B & $15.5$ & $24.7{\scriptstyle\pm 6.5}$ & $\mathbf{25.3}{\scriptstyle\pm 5.9}$ & $\mathbf{31.3}{\scriptstyle\pm 2.0}$ & $26.1{\scriptstyle\pm 8.8}$ \\
    Qwen3-VL-8B & $31.8$ & $30.8{\scriptstyle\pm 1.5}$ & $\mathbf{35.2}{\scriptstyle\pm 0.6}$ & $22.8{\scriptstyle\pm 4.0}$ & $\mathbf{32.3}{\scriptstyle\pm 0.9}$ \\
    \bottomrule
  \end{tabular*}
\end{table}

\begin{table}[!t]
  \caption{Reward computed only on the compound subset of the long-form VQA evaluation set: per-segment ROUGE-L F1}
  \label{app:vqa-compound}
  \centering
  \small
  \begin{tabular*}{\linewidth}{@{\extracolsep{\fill}}llccccccc@{}}
    \toprule
    & & & \multicolumn{3}{c}{GRPO} & \multicolumn{3}{c}{Dr.~GRPO} \\
    \cmidrule(lr){4-6}\cmidrule(lr){7-9}
    Model & Mode & Base & Holistic & SD-GRPO & $\Delta$ & Holistic & SD-GRPO & $\Delta$ \\
    \midrule
    Qwen3-VL-2B & Direct    & $27.3$ & $34.1$          & $\mathbf{36.6}$ & $+2.5$ & $33.0$          & $\mathbf{41.5}$ & $+8.5$ \\
                & Reasoning & $21.4$ & $30.9$          & $\mathbf{36.7}$ & $+5.8$ & $30.2$          & $\mathbf{32.9}$ & $+2.7$ \\
    \midrule
    Qwen3-VL-4B & Direct    & $38.4$ & $41.3$          & $\mathbf{43.5}$ & $+2.2$ & $42.8$          & $\mathbf{43.4}$ & $+0.6$ \\
                & Reasoning & $40.4$ & $39.8$          & $\mathbf{43.7}$ & $+3.9$ & $43.0$          & $\mathbf{45.4}$ & $+2.4$ \\
    \midrule
    Qwen3-VL-8B & Direct    & $45.2$ & $\mathbf{53.2}$ & $49.6$          & $-3.6$ & $51.9$          & $\mathbf{53.4}$ & $+1.5$ \\
                & Reasoning & $47.2$ & $50.9$          & $\mathbf{53.3}$ & $+2.4$ & $50.2$          & $\mathbf{53.0}$ & $+2.8$ \\
    \bottomrule
  \end{tabular*}
\end{table}

\section{Composite scientific figure captioning: Prompts and conversation format}
\label{app:mmsci-prompts}
The experiment design is visualized in Figure~\ref{fig:appendix}.

\begin{figure*}[t]
  \centering
  \includegraphics[width=\textwidth]{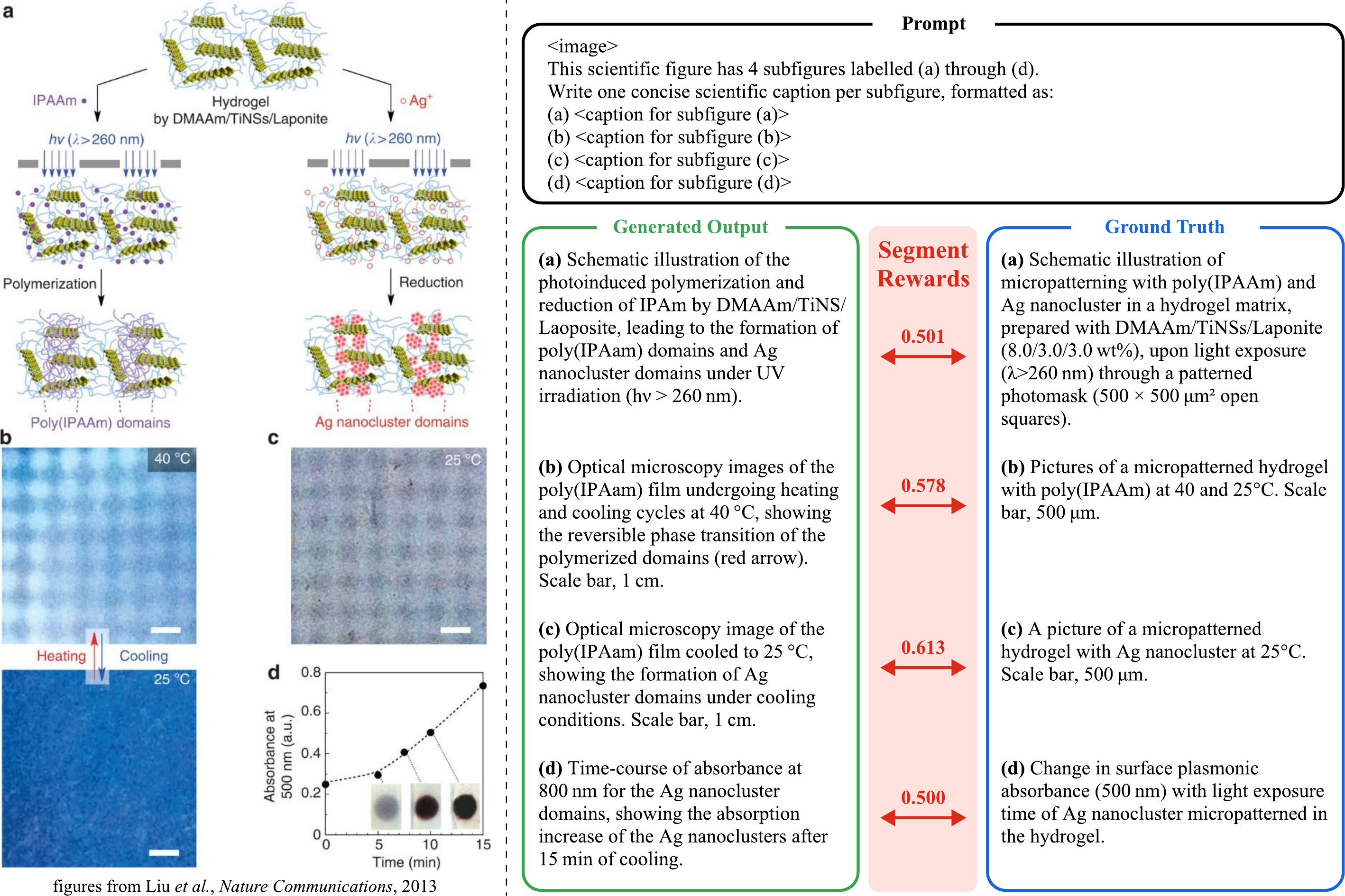}
    \caption{Visualization of the composite scientific figure captioning experiment. Segment rewards are computed by comparing generated output and ground truth by subcaptions.}
  \label{fig:appendix}
\end{figure*}

\paragraph{System prompt.}
We use Qwen3-VL's default system message (\texttt{"You are a helpful assistant."}) for all conditions; no task-specific system prompt is set.

\paragraph{User turn.}
The user turn carries one image attachment (the composite scientific figure) followed by an instruction enumerating the expected subcaption labels. The instruction is templated per-figure based on $S$ and the canonical label sequence $(a),(b),\ldots$:

\begin{tcolorbox}[breakable, title={User prompt---nothink (main results)},
                  fonttitle=\bfseries\small,
                  colback=gray!5, colframe=black,
                  colbacktitle=black, coltitle=white]
\small\ttfamily\noindent
[image]\\
This scientific figure has $S$ subfigures labeled $(a)$, $(b)$, \ldots, and $(\cdot)$.\\
Write one concise scientific caption per subfigure, formatted as:\\
$(a)$ <caption for subfigure $(a)$>\\
$(b)$ <caption for subfigure $(b)$>\\
\ldots
\end{tcolorbox}

\noindent In \textbf{think} mode (used only for the appendix experiments; see Appendix~\ref{app:think-results}) the prompt is replaced with a more structured variant that asks for a per-section reasoning block before each caption, with two illustrative examples in-line:

\begin{tcolorbox}[breakable, title={User prompt---think mode},
                  fonttitle=\bfseries\small,
                  colback=gray!5, colframe=black,
                  colbacktitle=black, coltitle=white]
\small\ttfamily\noindent
[image]\\
This scientific figure has $S$ subfigures labeled $(a)$, $(b)$, \ldots, and $(\cdot)$.\\
For each subfigure, output a <think>...</think> block followed by one caption line prefixed with the subfigure label. The think block is your private analysis; the caption line is the polished description that gets evaluated. They should serve different roles---do NOT restate the same content in both. Keep the caption to one concise sentence. Output the segments back-to-back, with no blank lines between them, in alphabetical order.\\
\\
Example output (for a 3-subfigure figure):\\
<think>Fluorescence signal grows from 0 to 24 h---the marker is being expressed in a time-dependent manner.</think>\\
$(a)$ Confocal images of the reporter at 0, 6, and 24 h, showing progressive accumulation.\\
<think>Bar heights track the imaging trend and the last two error bars don't overlap---the increase is real, not noise.</think>\\
$(b)$ Mean fluorescence intensity per cell across time points, with a significant rise by 24 h.\\
\ldots
\end{tcolorbox}

\paragraph{Single-turn decoding.}
All $S$ subcaptions are emitted in a single autoregressive pass within one assistant turn. The trainer does not inject any user follow-ups between subcaptions; the model is responsible for emitting the canonical $(a),(b),\ldots$ markers itself in sequence, separated by line breaks. Per-section spans are recovered \emph{post hoc} by parsing line-start $(letter)$ markers in the decoded response and mapping them to token positions via the tokenizer's offset map (Section~\ref{sec:method}). Per-section rewards are computed on the caption text only; in think mode any \texttt{<think>...</think>} block is included in the preceding section's token span (so it is loss-shaped by that section's advantage) but stripped before being passed to the SciBERT-F1 / BARTScore scorers.

\paragraph{Format gate.}
A response is \emph{format-valid} iff the line-start $(letter)$ sequence exactly matches the expected $(a),(b),\ldots,(z_S)$, with no duplicates, no extras, and no out-of-order markers. Invalid responses have their reward floored to $-1$ before advantage computation; this prevents reward hacking the per-section signal under SD-GRPO (a model that emits captions only in think mode would otherwise have format-invalid output but pristine per-section SciBERT-F1, leaving SD-GRPO no signal to honor the strict format).

\paragraph{Resulting response structure.}
Under direct mode (\texttt{max\_response\_tokens}=$1{,}024$), the response is a single contiguous block of $S$ caption lines:

\begin{tcolorbox}[breakable, title={Response---direct mode},
                  fonttitle=\bfseries\small,
                  colback=gray!5, colframe=black,
                  colbacktitle=black, coltitle=white]
\small\noindent
\textbf{User:} [composite figure image] + <prompt with $S$ labels>\\
\textbf{Assistant:}\\
\hspace*{2em}$(a)$ <model-generated caption for $(a)$>\\
\hspace*{2em}$(b)$ <model-generated caption for $(b)$>\\
\hspace*{2em}\ldots\\
\hspace*{2em}$(z_S)$ <model-generated caption for $(z_S)$>
\end{tcolorbox}

\noindent and under think mode (\texttt{max\_response\_tokens}=$2{,}048$), each caption line is preceded by a \texttt{<think>}\ldots\texttt{</think>} reasoning block:

\begin{tcolorbox}[breakable, title={Response---think mode},
                  fonttitle=\bfseries\small,
                  colback=gray!5, colframe=black,
                  colbacktitle=black, coltitle=white]
\small\noindent
\textbf{User:} [composite figure image] + <prompt with $S$ labels>\\
\textbf{Assistant:}\\
\hspace*{2em}<think> <reasoning for $(a)$> </think>\\
\hspace*{2em}$(a)$ <model-generated caption for $(a)$>\\
\hspace*{2em}<think> <reasoning for $(b)$> </think>\\
\hspace*{2em}$(b)$ <model-generated caption for $(b)$>\\
\hspace*{2em}\ldots
\end{tcolorbox}

\section{Composite scientific figure captioning: Detailed results}
\label{app:mmsci-results}

\subsection{Per-seed standard deviations for the main reward table}
\label{app:mmsci-train-std}

Table~\ref{tab:mmsci-train-std} reports the same training-reward metrics as Table~\ref{tab:mmsci-train} with sample standard deviations across the $3$ seeds. Cross-seed variance is small at every cell ($\le 0.003$ on concat reward and $\le 0.005$ on section-mean reward), confirming the rankings in the main table are stable to seed.

\begin{table}
  \caption{Composite scientific figure captioning, training-reward metrics with per-seed standard deviations (companion to Table~\ref{tab:mmsci-train}). Trained-method values are mean $\pm$ sample std across $3$ seeds; Base is a single zero-shot evaluation. Bold marks the best per row.}
  \label{tab:mmsci-train-std}
  \centering
  \small
  \begin{tabular*}{\linewidth}{@{\extracolsep{\fill}}llcccc@{}}
    \toprule
    Model & Metric & Base & Holistic & SD-GRPO & Blend \\
    \midrule
    Qwen3-VL-4B & Concat reward       & $0.3993$ & $0.4156_{\pm.003}$ & $0.4122_{\pm.001}$ & $\mathbf{0.4196_{\pm.001}}$ \\
                & Section-mean reward & $0.3819$ & $0.3915_{\pm.005}$ & $0.4038_{\pm.002}$ & $\mathbf{0.4041_{\pm.002}}$ \\
    \midrule
    Qwen3-VL-8B & Concat reward       & $0.4010$ & $0.4227_{\pm.002}$ & $0.4179_{\pm.002}$ & $\mathbf{0.4266_{\pm.002}}$ \\
                & Section-mean reward & $0.3835$ & $0.3986_{\pm.002}$ & $\mathbf{0.4114_{\pm.002}}$ & $0.4112_{\pm.002}$ \\
    \bottomrule
  \end{tabular*}
\end{table}

\subsection{Per-axis pairwise margins from the GPT-4o text-only judge}
\label{app:mmsci-pairwise}

Table~\ref{tab:mmsci-pairwise} expands the head-to-head pairwise margins from Table~\ref{tab:mmsci-judge} into per-axis numbers (accuracy, completeness, fluency, overall). Each cell is Blend's win\% minus the opponent's win\% on that axis; ties constitute the remainder of each comparison. Counts aggregate over $3$ seeds $\times$ $400$ figures ($1{,}200$ paired comparisons per cell), with position randomization per call.

\begin{table}
  \caption{Per-axis GPT-4o text-only pairwise margins (Blend's win\% $-$ opponent's win\%), aggregated over $3$ seeds $\times$ $400$ figures.}
  \label{tab:mmsci-pairwise}
  \centering
  \small
  \begin{tabular*}{\linewidth}{@{\extracolsep{\fill}}llcccc@{}}
    \toprule
    Scale & Opponent & Accuracy & Completeness & Fluency & Overall \\
    \midrule
    Qwen3-VL-4B & Holistic GRPO & $+10.2\%$ & $+10.8\%$ & $+7.6\%$  & $+10.5\%$ \\
                & SD-GRPO       & $+10.2\%$ & $+10.0\%$ & $+15.8\%$ & $+11.2\%$ \\
    \midrule
    Qwen3-VL-8B & Holistic GRPO & $+9.2\%$  & $+8.8\%$  & $+5.2\%$  & $+8.5\%$  \\
                & SD-GRPO       & $+16.7\%$ & $+16.7\%$ & $+18.4\%$ & $+17.5\%$ \\
    \bottomrule
  \end{tabular*}
\end{table}

\paragraph{Judge prompts.} Both judge views use \texttt{gpt-4o-2024-08-06} via the OpenAI Chat Completions API at \texttt{temperature}=$0$, with \texttt{response\_format}=\texttt{json\_object} and \texttt{max\_tokens}=$80$. The candidate and reference subcaption sets are passed as plain text only---the figure image is not provided.

\begin{tcolorbox}[breakable, title={System prompt---single-rater absolute},
                  fonttitle=\bfseries\small,
                  colback=gray!5, colframe=black,
                  colbacktitle=black, coltitle=white]
\small\ttfamily\noindent
You are an expert at evaluating scientific figure captions. Compare a candidate subcaption set against a human-written reference for the same multi-panel scientific figure. Both are formatted as $(a)$ \ldots $(b)$ \ldots $(c)$ \ldots blocks, one line per subfigure. Score the candidate on three axes using the rubric provided.
\end{tcolorbox}

\begin{tcolorbox}[breakable, title={User prompt---single-rater absolute},
                  fonttitle=\bfseries\small,
                  colback=gray!5, colframe=black,
                  colbacktitle=black, coltitle=white]
\small\ttfamily\noindent
This scientific figure has $S$ subfigures, labeled $(a)$, $(b)$, \ldots, and $(\cdot)$. The candidate must include every expected label; missing or extra labels are scoring failures, not formatting noise.\\
\\
Reference subcaptions (human-written, treat as ground truth):\\
<reference subcaption set>\\
\\
Candidate subcaptions (model output, to be scored):\\
<candidate subcaption set>\\
\\
Score the candidate on three axes, using a $1.0$--$5.0$ scale in $0.5$ increments:\\
\\
1. \textbf{Accuracy} -- does each subfigure caption correctly describe what the reference describes (apparatus, conditions, observed effect, statistical claims)? Penalize hallucinated or contradicted claims.\\
2. \textbf{Completeness} -- does the candidate cover the salient information the reference covers for each subfigure?\\
3. \textbf{Fluency} -- is the candidate's language scientific, concise, and grammatical?\\
\\
Anchor points (use $0.5$ steps between):\\
- $5.0$  near-publication-quality; matches the reference faithfully on this axis\\
- $4.0$  substantively correct, minor omissions or stylistic differences\\
- $3.0$  partially correct; salient details missing or slightly off\\
- $2.0$  multiple errors or substantial omissions across subfigures\\
- $1.0$  mostly wrong or incoherent\\
\\
Return ONLY a JSON object with these exact keys (floats $1.0$--$5.0$):\\
\{"accuracy": <float>, "completeness": <float>, "fluency": <float>\}
\end{tcolorbox}

\begin{tcolorbox}[breakable, title={System prompt---head-to-head pairwise},
                  fonttitle=\bfseries\small,
                  colback=gray!5, colframe=black,
                  colbacktitle=black, coltitle=white]
\small\ttfamily\noindent
You are an expert at evaluating scientific figure captions. You will see a human-written reference for a multi-panel scientific figure plus two candidate caption sets (Caption A and Caption B). Decide which candidate is better on three axes, allowing ties. Be impartial; the candidates' labels A/B are arbitrary.
\end{tcolorbox}

\begin{tcolorbox}[breakable, title={User prompt---head-to-head pairwise},
                  fonttitle=\bfseries\small,
                  colback=gray!5, colframe=black,
                  colbacktitle=black, coltitle=white]
\small\ttfamily\noindent
This scientific figure has $S$ subfigures, labeled $(a)$, $(b)$, \ldots, and $(\cdot)$. Each candidate must include every expected label; missing or extra labels are scoring failures, not formatting noise.\\
\\
Reference subcaptions (human-written, treat as ground truth):\\
<reference subcaption set>\\
\\
--- Caption A ---\\
<candidate A subcaption set>\\
\\
--- Caption B ---\\
<candidate B subcaption set>\\
\\
For each axis below, decide whether A is better, B is better, or they are tied:\\
\\
1. \textbf{Accuracy} -- which candidate more correctly describes each subfigure as the reference does (apparatus, conditions, observed effect, statistical claims)? Penalize hallucinated or contradicted claims.\\
2. \textbf{Completeness} -- which candidate covers more of the salient information present in the reference for each subfigure?\\
3. \textbf{Fluency} -- which candidate's language is more scientific, concise, and grammatical?\\
\\
Also give an overall preference: which candidate, all axes considered, would you prefer as the published caption?\\
\\
Return ONLY a JSON object with these exact keys (each value one of "A", "B", "tie"):\\
\{"accuracy": "A|B|tie", "completeness": "A|B|tie", "fluency": "A|B|tie", "overall": "A|B|tie"\}\\
\\
The order of A and B is randomised per call (we de-randomise post-hoc when tabulating margins).
\end{tcolorbox}

\subsection{Blend mixing-weight $\alpha$ ablation}
\label{app:mmsci-alpha}

We sweep the Blend mixing weight $\alpha$ that controls the SD-GRPO/Holistic balance in the advantage. All other hyperparameters match the main setup; we run 1 seed at each of $\alpha \in \{0.25, 0.5, 0.75\}$ and report mean values at final checkpoint in Table~\ref{tab:mmsci-alpha}.

\begin{table}
  \caption{Blend mixing-weight $\alpha$ sweep on Qwen3-VL-4B (values for $1$ seed). $\alpha$ is the SD-GRPO weight in the advantage blend. Bold marks the best per column.}
  \label{tab:mmsci-alpha}
  \centering
  \small
  \begin{tabular*}{\linewidth}{@{\extracolsep{\fill}}lccc@{}}
    \toprule
    Mixing weight & Concat reward & Section-mean reward & SciBERT-F1 (concat) \\
    \midrule
    $\alpha = 0.25$ (more Holistic) & $0.4183$ & $0.4008$ & $0.6331$ \\
    $\alpha = 0.50$ (default)       & $\mathbf{0.4185}$ & $0.4041$ & $\mathbf{0.6356}$ \\
    $\alpha = 0.75$ (more SD-GRPO)  & $\mathbf{0.4185}$ & $\mathbf{0.4082}$ & $0.6329$ \\
    \bottomrule
  \end{tabular*}
\end{table}

\subsection{Mean-GRPO ablation: averaging vs.\ z-norm}
\label{app:mmsci-holistic-avg}
Table~\ref{tab:mmsci-holistic-avg} gives the results for \emph{mean-GRPO}: it lands at section-mean reward $0.392$, well below SD-GRPO and Blend. The ablation rules out the simpler hypothesis that Blend's per-section gains come from \emph{averaging} per-section signals into a richer scalar; what is load-bearing is the per-section \emph{variance-normalization} that SD-GRPO contributes, since broadcasting an averaged scalar to all response tokens delivers no per-section gradient direction.

\begin{table}
  \caption{Mean-GRPO ablation on Qwen3-VL-4B-Instruct (mean $\pm$ std over $3$ seeds). \emph{Mean-GRPO} replaces the concat-level scalar with the per-section mean while keeping the broadcast advantage estimator all other hyperparameters match the main setup.}
  \label{tab:mmsci-holistic-avg}
  \centering
  \small
  \begin{tabular*}{\linewidth}{@{\extracolsep{\fill}}lccc@{}}
    \toprule
    Method & Concat reward & Section-mean reward & SciBERT-F1 (concat) \\
    \midrule
    Holistic GRPO       & $0.4156_{\pm.003}$ & $0.3916_{\pm.005}$ & $0.6305_{\pm.002}$ \\
    Mean-GRPO        & $0.4031_{\pm.002}$ & $0.3923_{\pm.005}$ & $0.6165_{\pm.003}$ \\
    SD-GRPO             & $0.4122_{\pm.001}$ & $0.4038_{\pm.002}$ & $0.6254_{\pm.001}$ \\
    Blend               & $\mathbf{0.4196_{\pm.001}}$ & $\mathbf{0.4041_{\pm.002}}$ & $\mathbf{0.6356_{\pm.001}}$ \\
    \bottomrule
  \end{tabular*}
\end{table}

\subsection{Cross-metric checks}
\label{app:mmsci-cross-metrics}

Table~\ref{tab:mmsci-cross-metrics} reports surface and semantic metrics---ROUGE-L, BERTScore-F1 (RoBERTa-large, layer 17), and CIDEr---under two aggregation views: \emph{concat} (one (pred, ref) pair per figure; $n{=}400$) and \emph{section-mean} (one pair per labeled section; $n{=}1{,}418$). Numbers are means across $3$ seeds at the 4B scale. Blend wins or ties on every cell, with the largest gaps appearing on the per-subfigure view---consistent with the per-section gradient direction that SD-GRPO and Blend share but Holistic does not. The mean-GRPO ablation tracks plain Holistic on ROUGE-L and CIDEr and is slightly lower on BERTScore-F1, again confirming that per-section averaging alone does not match the per-section z-norm contribution.

\begin{table}[!t]
  \caption{Cross-metric checks for the composite scientific figure captioning experiments (mean over $3$ seeds).}
  \label{tab:mmsci-cross-metrics}
  \centering
  \small
  \begin{tabular*}{\linewidth}{@{\extracolsep{\fill}}lcccccc@{}}
    \toprule
    & \multicolumn{3}{c}{Concat} & \multicolumn{3}{c}{Section-mean} \\
    \cmidrule(lr){2-4}\cmidrule(lr){5-7}
    Method & ROUGE-L & BERT-F1 & CIDEr & ROUGE-L & BERT-F1 & CIDEr \\
    \midrule
    Holistic GRPO       & $21.0$ & $83.9$ & $\mathbf{4.8}$ & $19.0$ & $84.5$ & $8.6$ \\
    Mean-GRPO & $21.0$ & $83.5$ & $3.7$          & $19.2$ & $84.6$ & $8.0$ \\
    SD-GRPO             & $21.7$ & $83.8$ & $3.8$          & $19.9$ & $\mathbf{85.0}$ & $8.3$ \\
    \textbf{Blend}      & $\mathbf{22.0}$ & $\mathbf{84.0}$ & $4.7$ & $\mathbf{20.2}$ & $\mathbf{85.0}$ & $\mathbf{9.4}$ \\
    \bottomrule
  \end{tabular*}
\end{table}

\section{Additional reasoning mode results}
\label{app:think-results}

\subsection{Multi-panel dense image captioning: Reasoning models}
\label{app:docci-think-4b}

We additionally trained Qwen3-VL-4B-Instruct on the multi-panel dense image captioning task of Section~\ref{sec:docci-exp} under per-segment reasoning-mode prompting (a brief \texttt{<think>}\ldots\texttt{</think>} step before each panel description). Despite the lower performance compared to direct mode, the SD-GRPO advantage from the main experiment carries over: SD-GRPO outperforms Holistic GRPO (Table~\ref{tab:docci-think-4b}).

\begin{table}[!t]
  \caption{Multi-panel dense image captioning under per-segment reasoning-mode prompting using Qwen3-VL-4B-Instruct. Values are BERTScore-F1 (per-panel mean), mean $\pm$ std over three seeds.}
  \label{tab:docci-think-4b}
  \centering
  \small
  \begin{tabular*}{\linewidth}{@{\extracolsep{\fill}}lccc@{}}
    \toprule
    $S$ & Holistic & SD-GRPO & $\Delta$ \\
    \midrule
    3 & $0.313_{\pm 0.007}$ & $\mathbf{0.342}_{\pm 0.011}$ & $+0.029$ \\
    5 & $0.279_{\pm 0.001}$ & $\mathbf{0.323}_{\pm 0.004}$ & $+0.044$ \\
    7 & $0.285_{\pm 0.010}$ & $\mathbf{0.297}_{\pm 0.008}$ & $+0.012$ \\
    \bottomrule
  \end{tabular*}
\end{table}

\begin{table}[!t]
  \caption{Results of composite scientific figure captioning with per-segment-think prompt using Qwen3-VL-4B-Instruct models. Values are mean over $2$ seeds.}
  \label{tab:mmsci-think-4b}
  \centering
  \small
  \begin{tabular*}{\linewidth}{@{\extracolsep{\fill}}lccc@{}}
    \toprule
    Method & Concat reward & Section-mean reward & SciBERT-F1 (concat) \\
    \midrule
    Holistic GRPO & $0.401$          & $0.381$          & $0.626$ \\
    SD-GRPO       & $0.406$          & $\mathbf{0.402}$ & $0.622$ \\
    Blend         & $\mathbf{0.415}$ & $\mathbf{0.402}$ & $\mathbf{0.633}$ \\
    \bottomrule
  \end{tabular*}
\end{table}

\subsection{Composite scientific figure captioning: Reasoning models}
\label{app:mmsci-think-4b}

We additionally trained Qwen3-VL-4B-Instruct with a per-segment-think prompt that asks the model to emit a \texttt{<think>}\ldots\texttt{</think>} block before each labelled subcaption (Appendix~\ref{app:mmsci-prompts}). The model converges to producing one think block per section (avg \texttt{num\_thinks} ${\approx} 3.5$, matching avg $S$), so the per-section z-norm of SD-GRPO has well-defined per-segment token spans to attribute reward to. Despite the format working, every method scores lower on every metric than its nothink counterpart in Table~\ref{tab:mmsci-train}---consistent with our broader finding that the model's reasoning content does not translate into higher caption quality on this task. The Blend $>$ Holistic, SD-GRPO ranking from the main table is preserved (Table~\ref{tab:mmsci-think-4b}).


\end{document}